\documentclass[10pt, a4paper]{article}
\usepackage[utf8]{inputenc}
\usepackage[T1]{fontenc}  %
\usepackage[margin=.8in]{geometry}
\usepackage{graphicx} 
\usepackage{amsmath}
\usepackage{amssymb}
\usepackage{amsthm}
\usepackage{algorithm}
\usepackage{algpseudocode}
\usepackage{enumitem}
\usepackage{booktabs}
\usepackage{multirow}
\usepackage{multicol}
\usepackage{subcaption}
\usepackage{moreverb,url}
\usepackage{authblk}

\theoremstyle{plain}          
\newtheorem{theorem}{Theorem}[section]   
\newtheorem{lemma}[theorem]{Lemma}       
\newtheorem{corollary}[theorem]{Corollary}

\theoremstyle{definition}     

\theoremstyle{remark}         

\title{Efficient Incremental SLAM via Information-Guided and Selective Optimization}
\author{Reza Arablouei}
\affil{\small Commonwealth Scientific and Industrial Research Organisation (CSIRO)\\
       Pullenvale QLD 4069, Australia}
\date{}

\begin{document}

\maketitle

\begin{abstract}
We present an efficient incremental SLAM back-end that achieves the accuracy of full batch optimization while substantially reducing computational cost. The proposed approach combines two complementary ideas: information-guided gating (IGG) and selective partial optimization (SPO). IGG employs an information-theoretic criterion based on the log-determinant of the information matrix to quantify the contribution of new measurements, triggering global optimization only when a significant information gain is observed. This avoids unnecessary relinearization and factorization when incoming data provide little additional information. SPO executes multi-iteration Gauss-Newton (GN) updates but restricts each iteration to the subset of variables most affected by the new measurements, dynamically refining this active set until convergence. Together, these mechanisms retain all measurements to preserve global consistency while focusing computation on parts of the graph where it yields the greatest benefit. We provide theoretical analysis showing that the proposed approach maintains the convergence guarantees of full GN. Extensive experiments on benchmark SLAM datasets show that our approach consistently matches the estimation accuracy of batch solvers, while achieving significant computational savings compared to conventional incremental approaches. The results indicate that the proposed approach offers a principled balance between accuracy and efficiency, making it a robust and scalable solution for real-time operation in dynamic data-rich environments.
\end{abstract}


\section{Introduction}

Simultaneous localization and mapping (SLAM) is a fundamental capability for autonomous robots, enabling the continuous estimation of both a robot’s pose and the surrounding environment map. Modern SLAM systems often adopt graph-based nonlinear optimization formulations, in which robot poses are represented as nodes and spatial constraints as edges in a factor graph~\cite{262373,Frese2006,5681215}. Compared to filtering-based methods, this full smoothing approach typically yields higher accuracy and improved consistency~\cite{dellaert2006squaresam}, but also leads to a continually expanding estimation problem. In long-term, data-rich deployments, the number of poses and measurements can grow without bound, substantially increasing memory and computation requirements. Repeatedly solving the full SLAM problem from scratch becomes prohibitive for real-time applications, motivating extensive research into incremental SLAM methods that update the solution efficiently as new data arrive~\cite{938382}.

Early SLAM algorithms were predominantly based on recursive filtering, with the extended Kalman filter (EKF) as a widely used example. Although EKF-SLAM supports real-time operation, its dense covariance representation leads to $\mathcal{O}(N^2)$ complexity, and accumulated linearization errors can degrade map consistency over time. Delayed-state filters mitigate computational load by maintaining only a fixed-size window of recent keyframes~\cite{eustice2006exactly}, but as the window grows to preserve accuracy, costs again become significant. In contrast, smoothing-based SLAM formulates the problem as a nonlinear least-squares optimization over a pose graph~\cite{factor_graphs_for_robot_perception}, exploiting the sparsity of the underlying information matrix.

Early smoothing and mapping (SAM) approaches, such as Square Root SAM~\cite{dellaert2006squaresam}, demonstrated that solving the full SLAM problem via sparse linear algebra (e.g., QR or Cholesky factorization of the information matrix) improves both accuracy and efficiency compared to EKF-based methods. Open-source frameworks such as g2o~\cite{kummerle2011g2o} and GTSAM~\cite{gtsam} have further streamlined deployment by providing efficient graph construction, factorization, and solver interfaces. Beyond direct solvers, iterative approaches~\cite{dellaert2010spcg,jian2013supportprecon,jian2014ispcg} leverage sparse matrix–vector products to reduce computational cost, and specialized fast pose-graph optimizers~\cite{olson2006fast} can handle poor initial estimates effectively.

Among incremental SAM algorithms, iSAM~\cite{kaess2008isam} is a landmark contribution, maintaining the square-root information matrix and updating it with new measurements via low-rank sparse matrix factor updates~\cite{gill1974design}. This enables real-time updates while reusing previously computed structure. However, the original iSAM required periodic batch relinearization and variable reordering to maintain consistency, effectively performing occasional full optimizations to correct linearization drift.
Its successor, iSAM2~\cite{kaess2012isam2}, addressed these limitations by introducing the Bayes tree, a junction-tree-based data structure that supports fluid relinearization and incremental variable reordering. iSAM2 uses a threshold-based wildfire strategy to relinearize only variables affected by new information, eliminating the need for expensive batch resets while preserving accuracy.

Other advances in incremental SLAM include robust solvers and alternative inference strategies. The RISE algorithm~\cite{rosen2014rise} incorporates a trust-region method (Powell’s dog-leg) into an incremental framework, improving robustness to strong nonlinearities and ill-conditioned systems while maintaining speeds comparable to Gauss-Newton (GN) methods. NF-iSAM~\cite{huang2021nfiSAM} extends incremental smoothing to non-Gaussian estimation by using normalizing flows to represent arbitrary posteriors, retaining the sparsity and efficiency of iSAM2’s Bayes tree under non-Gaussian measurement models. Robustness has also been pursued in riSAM~\cite{mcgann2023risam}, which leverages graduated non-convexity~\cite{kang2024egnc,yang2020gnc,choi2023adaptiveGNC} to handle outliers and non-convexity. In parallel, works such as incremental Cholesky factorization~\cite{polok2013incrementalChol} and AprilSAM~\cite{wang2018aprilsam} have focused on improving update efficiency through algorithmic refinements that accelerate matrix updates. Collectively, these developments reflect the SLAM community’s drive toward back-ends that are both online-efficient and robust under real-world conditions.

Despite these advances, scalable SLAM in highly dynamic, data-rich environments remains challenging. Robots operating for extended durations or equipped with high-frequency sensors (e.g., dense visual or LiDAR) can receive a stream of measurements, many redundant or only marginally informative. Incorporating every measurement into the pose graph and triggering a full solver update at each increment is computationally wasteful and may degrade real-time performance.
To address this, researchers have explored measurement sparsification or selection based on information content. Several works~\cite{Khosoussi2020,Kretzschmar2012,5325904} employ information-theoretic criteria to design sparse yet reliable pose graphs that approximate the estimation quality of the full graph, selecting a near-D-optimal subset of loop closures or constraints to maximize the determinant of the Fisher information matrix.

Approaches to information-driven graph sparsification and active SLAM~\cite{activeSLAMsurvey} have used well-defined quality measures such as the log-determinant of the Fisher information matrix (D-optimality) and the algebraic connectivity (Fiedler eigenvalue of the graph Laplacian). Examples include information-theoretic loop-closure detection~\cite{carlevaris2012information}, factor-based node marginalization~\cite{carlevaris2013generic}, and conservative edge sparsification~\cite{carlevaris2014conservative}. Similarly, Doherty \textit{et al.}~\cite{doherty2022spectralSparsify} propose a spectral sparsification method that retains edges maximizing algebraic connectivity, a property correlated with SLAM accuracy. By maximizing the Fiedler value of the measurement graph, their method preserves global consistency while reducing stored constraints.
These studies show that significant computational savings are possible through graph pruning or edge selection guided by principled information measures. However, permanently discarding measurements risks loss of consistency or robustness if the selection is imperfect~\cite{doherty2024MAC}. Moreover, many sparsification methods operate offline or as a separate preprocessing step, rather than being tightly integrated into the live optimization loop to decide, in real time, whether a new measurement should be incorporated.

\subsection{Contributions}

We propose an efficient incremental SLAM framework that integrates information-based variable selection directly into the optimization back-end, thereby avoiding unnecessary computations while preserving global consistency. The proposed approach comprises two key components:  
(i) information-guided gating (IGG), and (ii) selective partial optimization (SPO).  
Beyond algorithmic innovations, we also establish theoretical guarantees, showing that the proposed heuristics preserve the convergence properties of full GN optimization. In summary, our main contributions are:

\begin{itemize}
\item \emph{Information-Guided Gating (IGG)}: An information-theoretic mechanism that monitors the change in the log-determinant of the information matrix to evaluate the contribution of new measurements. Only when the predicted information gain exceeds a threshold is a global update triggered. Otherwise, optimization is restricted to local variables directly affected by the new measurements. 
\item \emph{Selective Partial Optimization (SPO)}: A multi-iteration nonlinear GN solver that, at each iteration, updates and relinearizes only the variables that have not yet converged. This subset is dynamically refined based on convergence thresholds and graph connectivity, focusing computation where it yields the greatest gains.
\item \emph{Theoretical Guarantees}: Rigorous analysis demonstrating that the proposed approach combining IGG and SPO converges to the same stationary point and achieves the same local rate as full GN under standard smoothness assumptions, thereby providing principled support for the observed efficiency gains.
\end{itemize}

Unlike sparsification methods, which may discard information and compromise global consistency, our approach retains all measurements while limiting optimization to the most impactful updates. The result is a robust and scalable SLAM back-end that delivers batch-level accuracy at a fraction of the per-update computational cost. In contrast to iSAM2, which heuristically limits relinearization yet still solves the full global system at every increment, our approach provides a principled alternative: by unifying information-guided gating with selective partial optimization, it achieves the same convergence guarantees and accuracy of full GN while directing computation strictly to variables most affected by new information.

\section{Background}

SLAM involves the joint estimation of a robot’s trajectory and a map of the environment. Modern SLAM back-ends typically formulate this as a large-scale nonlinear least-squares optimization over a \emph{pose graph}, where nodes represent robot poses (and possibly landmarks) and edges represent spatial measurements between them.  
Consider a set of measurement residuals  
\begin{equation}
\mathbf{r}_j = \mathbf{m}_j - \mathbf{f}_j(\mathbf{x}_{\mathcal{V}_j}),
\end{equation}  
where \(\mathbf{m}_j\) is the \(j\)-th observed measurement (e.g., a relative pose or landmark observation), \(\mathbf{f}_j(\cdot)\) is the measurement prediction function, and \(\mathbf{x}_{\mathcal{V}_j}\) denotes the subset of state variables involved in measurement \(j\) (indexed by \(i\)). 
Assuming Gaussian measurement noise with covariance \(\boldsymbol\Sigma_j\), the maximum a posteriori (MAP) estimate can be obtained by minimizing the cost function  
\begin{equation}\label{cost}
c(\mathbf{x}) = \frac{1}{2}\sum_{j=1}^M \|\mathbf{r}_j\|^2_{\boldsymbol\Sigma_j} 
= \frac{1}{2}\sum_{j=1}^M \|\mathbf{m}_j - \mathbf{f}_j(\mathbf{x}_{\mathcal{V}_j})\|^2_{\boldsymbol\Sigma_j},
\end{equation}  
where \(\|\mathbf{r}_j\|^2_{\boldsymbol\Sigma_j} \triangleq \mathbf{r}_j^\mathsf{T} \boldsymbol\Sigma_j^{-1} \mathbf{r}_j\).  

This nonlinear least-squares problem is typically solved using iterative methods such as the GN algorithm or its Levenberg-Marquardt variant. Starting from an initial guess, GN linearizes the measurement functions around the current state estimate, producing the linearized normal equations
\begin{equation}
(\mathbf{J}^\mathsf{T} \mathbf{J}) \mathbf{d} = \mathbf{J}^\mathsf{T} \mathbf{r},
\end{equation}  
where \(\mathbf{J}\) is the stacked Jacobian of all residuals with respect to the full state vector \(\mathbf{x}\), and \(\mathbf{r}\) is the stacked residual vector. Solving for \(\mathbf{d}\) yields the state increment, and the estimate is updated as \(\mathbf{x} \gets \mathbf{x} + \mathbf{d}\). This process is repeated (with relinearization at each step) until convergence.  

Graph-based SLAM problems exhibit strong structural sparsity that can be exploited for computational efficiency. The information matrix (the approximate Hessian) \(\mathbf{H} = \mathbf{J}^\mathsf{T} \mathbf{J}\) is typically large but sparse and block-structured, reflecting the local connectivity of the pose graph: each variable \(i\) (pose or landmark) is directly linked to only a few others through shared measurements \(j\).  
In pure exploration scenarios without loop closures, the sparsity pattern is approximately block-tridiagonal. When loop closures occur, long-range links introduce additional nonzero blocks (fill-in), but the matrix remains sparse. Furthermore, in many SLAM problems, the diagonal blocks (self-information from priors or odometry) dominate the off-diagonal terms, leading to a form of near block-diagonal dominance. 
This sparsity implies that naive dense linear algebra is inefficient: factorizing a dense \(N \times N\) matrix costs \(\mathcal{O}(N^3)\), whereas exploiting sparsity can reduce the complexity dramatically. Efficient SLAM back-ends therefore rely on sparse matrix factorization or iterative methods to leverage this structure.

In practice, the linear system arising in each GN iteration can be solved either (i) \emph{directly}, through sparse matrix factorization (e.g., Cholesky or QR) followed by forward/backward substitution, or (ii) \emph{iteratively}, using solvers such as conjugate gradient that exploit sparse matrix-vector products. In this work, we adopt a direct sparse factorization approach because it offers high numerical accuracy, facilitates efficient reuse of factors in incremental updates, and enables rapid computation of information-theoretic quantities (e.g., log-determinants) that are essential for our IGG strategy.

\section{Proposed Approach}

The proposed approach, summarized in Algorithm 1, is an incremental nonlinear least-squares optimizer for SLAM that integrates two key ideas of information-guided gating (IGG) and selective partial optimization (SPO).
We formulate SLAM as a pose-graph optimization problem and maintain, throughout execution, the sparse Cholesky factor $\mathbf{R}$ of the information matrix $\mathbf{H}=\mathbf{J}^\mathsf{T}\mathbf{J}$. The algorithm operates in discrete increments, where each increment corresponds to the arrival of one or more new measurements. The framework supports both batch mode (processing multiple measurements received in rapid succession) and streaming mode (processing measurements one at a time). At each increment $t$, the algorithm determines which variables to update and how the update is performed, according to the following procedure.

\subsection{Graph Update and Initial Linearization}

When new measurements arrive, We first update the pose graph $\mathcal{G}_t$ by adding the corresponding edges, which connect the relevant robot poses or landmarks. If these edges introduce new poses or landmarks (and hence new variables) we also append their initial estimates to the state vector. We denote the total number of state variables after incorporating the new measurements as $N_t$. 
Subsequently, We linearize the new edges about the current state estimate, $\mathbf{x}_{t-1}$. Specifically, we compute the Jacobians of the new measurements with respect to the involved variables, evaluated at their current estimates. This process augments the system by adding new rows to the Jacobian $\mathbf{J}_{t}$ and new entries to the residual vector $\mathbf{r}_{t}$, and, if new variables are present, by adding corresponding columns to $\mathbf{J}_t$.

To efficiently incorporate these changes, we perform a low-rank Cholesky update to extend the existing factor $\mathbf{R}_{t-1}$ to $\mathbf{R}_{t}$. This produces the updated linear system $(\mathbf{R}_t^\mathsf{T}\mathbf{R}_t) \mathbf{d}_t = \mathbf{b}_t$, where $\mathbf{b}_t = \mathbf{J}_t^\mathsf{T} \mathbf{r}_t$. To preserve sparsity and reduce fill-in during factorization, we incrementally update the variable ordering $\boldsymbol{\pi}_t$ using the constrained column approximate minimum degree (CCOLAMD) algorithm~\cite{davis2006direct}. In addition, we incrementally update the elimination tree associated with $\mathbf{R}_{t}$, denoted by $\mathbf{p}_t$, to reflect changes in the graph structure.

\subsection{Information-Guided Gating}\label{sec:gating}

We define an information-theoretic quantity $\eta_t$ to measure the system’s average information content after incorporating the new measurements. Specifically, let
\begin{equation}
\eta_t = \sum_{i=1}^{N_t} \ln |\rho_{t,i}|,
\end{equation}
where $\rho_{t,i}$ is the $i$-th diagonal entry of $\mathbf{R}_t$. Since, we have
\begin{equation}
\sum_i \ln |\rho_{t,i}| = \ln \det(\mathbf{R}_t) = \frac{1}{2}\ln\det(\mathbf{J}_t^\mathsf{T}\mathbf{J}_t),
\end{equation}
$\eta_t$ is equal to half the log-determinant of the information matrix $\mathbf{H}_t=\mathbf{J}_t^\mathsf{T}\mathbf{J}_t$, which serves as a proxy for the information content or, equivalently, the negative log of the uncertainty volume~\cite{trefethen1997numerical}.

To isolate the effect of the new measurements from the trivial growth in $\eta_t$ caused by adding new variables, we compute the detrended change
\begin{equation}\label{delta_e}
\Delta \eta_t = \eta_t - \frac{N_{t-1}}{N_t} \eta_{t-1}
\end{equation}
and compare it against a preset threshold $\tau_\eta$. if $\Delta \eta_t< \tau_\eta$, he added measurements do not significantly reduce the system’s uncertainty. In this case, only the variables directly involved in the new edges are marked as \emph{potentially affected}, and their indices are stored in the set $\mathcal{S}_t$. This avoids unnecessary re-optimization in scenarios such as pure odometry updates with negligible incremental information.

\begin{algorithm}[!t]
\caption{Incremental SLAM with Information-Guided Gating and Selective Partial Optimization}
\label{alg:selective_update}
\small
\begin{algorithmic}[1]
\Require \parbox[t]{0.85\linewidth}{initial estimate $\mathbf{x}_0$,\\
        thresholds $\tau_\eta$ (information gain), $\tau_d$ (increment magnitude), $\tau_{\mathrm{GN}}$ (max GN iterations)}
\Ensure updated estimate $\mathbf{x}_t$ after each increment $t$
\For{$t = 1, 2, \dots$}
    \State incorporate new measurements into $\mathcal{G}_{t-1}$ to form $\mathcal{G}_t$
    \State $\mathcal{S}_t \gets$ variables involved in new measurements
    \State update $\mathbf{J}_t$, $\mathbf{r}_t$, $\mathbf{R}_t$, $\mathbf{b}_t$, $\boldsymbol{\pi}_t$, and $\mathbf{p}_t$ incrementally based on new measurements
    \State $\eta_t \gets \sum_{i=1}^{N_t} \ln|\mathrm{diag}(\mathbf{R}_t)|$
    \If{$\eta_t - \frac{N_{t-1}}{N_t} \eta_{t-1} > \tau_\eta$}
        \State $\mathcal{S}_t \gets \{1,\dots,N_t\}$
    \EndIf
    \For{$i_{\mathrm{GN}}=1,2,\dots,\tau_{\mathrm{GN}}$}
        \State solve $(\mathbf{R}_t^\mathsf{T}\mathbf{R}_t) \mathbf{d} = \mathbf{b}_t$ over $\mathcal{S}_t$ to obtain $\mathbf{d}_{\mathcal{S}_t}$
        \State update $\mathcal{S}_t$ based on $|\mathbf{d}_{\mathcal{S}_t}|>\tau_d$ and the connectivity in $\mathcal{G}_t$
        \If{$\mathcal{S}_t = \emptyset$}
            break
        \EndIf
        \State $\mathbf{x}_{\mathcal{S}_t} \gets \mathbf{x}_{\mathcal{S}_t} - \mathbf{d}_{\mathcal{S}_t}$
        \State update $\mathbf{J}_t$, $\mathbf{r}_t$, $\mathbf{R}_t$, and $\mathbf{b}_t$ for edges involving $\mathcal{S}_t$
    \EndFor
\EndFor
\end{algorithmic}
\end{algorithm}

Conversely, if $\Delta \eta_t \ge \tau_\eta$, the new measurements are deemed sufficiently informative to potentially influence the entire graph, for example, in the case of a loop closure or a high-accuracy pose prior. We therefore set $\mathcal{S}_t = \{1,2,\dots,N_t\}$, marking all variables as \emph{potentially affected}. While it is possible to attempt a more selective choice (e.g., restricting to the connected component impacted by the new edges), in typical SLAM graphs the system forms a single connected component, and determining the exact subset of globally affected variables is itself computationally expensive, potentially negating the benefits of selectivity.

Note that at this stage we identify only the \emph{potentially affected} variables. The subset of \emph{actually affected} variables (those updated and relinearized) is determined later, after the partial solve stage described in Section~\ref{sec:partial_solve}, via the procedure in Section~\ref{find_affected}.

\subsection{Selective Partial Solve}\label{sec:partial_solve}

Given the current set of \emph{potentially affected} variables $\mathcal{S}_t$ (identified from the IGG step or carried over from the previous GN iteration), we solve the linear system arising in the GN update:
\begin{equation}\label{sle}
(\mathbf{R}_t^\mathsf{T}\mathbf{R}_t)\mathbf{d}_t = \mathbf{b}_t.
\end{equation}
If $\mathcal{S}_t$ contains all variables, \eqref{sle} is solved by standard forward-backward substitution: 
\begin{equation}
\mathbf{R}_t^\mathsf{T}\mathbf{y} = \mathbf{b}_t,\quad \mathbf{R}_t\mathbf{d}_t = \mathbf{y}.
\end{equation}
When $|\mathcal{S}_t|\ll N_t$, it is more efficient to restrict computation to the dynamic subset $\mathcal{S}_t$ while treating the remaining variables $\mathcal{U}_t$ as static.
For clarity, we temporarily drop the increment index $t$ and denote these sets simply as $\mathcal{S}$ (dynamic) and $\mathcal{U}$ static, with $\mathcal{U}=\{1,..,N\}\setminus\mathcal{S}$. While the global variable ordering remains fixed by the factorization, we locally permute rows and columns of $\mathbf{R}$ and $\mathbf{b}$ to form the block structure
\begin{equation}
\mathbf{R} = 
\begin{bmatrix}
\mathbf{R}_{\mathcal{U}\mathcal{U}} & \mathbf{R}_{\mathcal{U}\mathcal{S}} \\
\mathbf{0} & \mathbf{R}_{\mathcal{S}\mathcal{S}}
\end{bmatrix},
\mathbf{b} = \begin{bmatrix}
\mathbf{b}_{\mathcal{U}} \\
\mathbf{b}_{\mathcal{S}}
\end{bmatrix},\text{ and }
\mathbf{b} = \begin{bmatrix}
\mathbf{d}_{\mathcal{U}} \\
\mathbf{d}_{\mathcal{S}}
\end{bmatrix}.
\end{equation}
The linear system $(\mathbf{R}^\mathsf{T}\mathbf{R})\mathbf{d} = \mathbf{b}$ then takes the block form:
\begin{equation}\label{blockNE}
\begin{bmatrix}
\mathbf{R}_{\mathcal{U}\mathcal{U}}^\mathsf{T}\mathbf{R}_{\mathcal{U}\mathcal{U}} & \mathbf{R}_{\mathcal{U}\mathcal{U}}^\mathsf{T}\mathbf{R}_{\mathcal{U}\mathcal{S}} \\[6pt]
\mathbf{R}_{\mathcal{U}\mathcal{S}}^\mathsf{T}\mathbf{R}_{\mathcal{U}\mathcal{U}} & \mathbf{R}_{\mathcal{S}\mathcal{S}}^\mathsf{T}\mathbf{R}_{\mathcal{S}\mathcal{S}} + \mathbf{R}_{\mathcal{U}\mathcal{S}}^\mathsf{T}\mathbf{R}_{\mathcal{U}\mathcal{S}}
\end{bmatrix}
\begin{bmatrix}
\mathbf{d}_{\mathcal{U}} \\[6pt]
\mathbf{d}_{\mathcal{S}}
\end{bmatrix}
=
\begin{bmatrix}
\mathbf{b}_{\mathcal{U}} \\[6pt]
\mathbf{b}_{\mathcal{S}}
\end{bmatrix}.
\end{equation}
Since the unaffected (static) variables remain fixed (i.e., $\mathbf{d}_{\mathcal{U}} = \mathbf{0}$), the first block row of \eqref{blockNE} becomes
\begin{equation}\label{bU}
\mathbf{R}_{\mathcal{U}\mathcal{U}}^\mathsf{T}\mathbf{R}_{\mathcal{U}\mathcal{S}}\mathbf{d}_{\mathcal{S}} = \mathbf{b}_{\mathcal{U}}.
\end{equation}
By introducing the intermediate vector
\begin{equation}\label{yU}
\mathbf{y}_\mathcal{U}=\mathbf{R}_{\mathcal{U}\mathcal{S}}\mathbf{d}_{\mathcal{S}},
\end{equation}
we can transform \eqref{bU} into the triangular linear system
\begin{equation}
\mathbf{R}_{\mathcal{U}\mathcal{U}}^\mathsf{T}\mathbf{y}_{\mathcal{U}} = \mathbf{b}_{\mathcal{U}}
\end{equation}
This equation can be solved efficiently by exploiting the cached $\mathbf{R}_{\mathcal{U}\mathcal{U}}$.
Substituting \eqref{yU} into the second block row of \eqref{blockNE} gives
\begin{align}
(\mathbf{R}_{\mathcal{S}\mathcal{S}}^\mathsf{T}\mathbf{R}_{\mathcal{S}\mathcal{S}})\mathbf{d}_{\mathcal{S}} &= \mathbf{b}_{\mathcal{S}} - (\mathbf{R}_{\mathcal{U}\mathcal{S}}^\mathsf{T}\mathbf{R}_{\mathcal{U}\mathcal{S}})\mathbf{d}_{\mathcal{S}}\\
&= \mathbf{b}_{\mathcal{S}} - \mathbf{R}_{\mathcal{U}\mathcal{S}}^\mathsf{T}\mathbf{y}_{\mathcal{U}},
\end{align}
Here, the term $-\mathbf{R}_{\mathcal{U}\mathcal{S}}^\mathsf{T}\mathbf{y}_{\mathcal{U}}$ represents the correction contributed by the static block’s equations. Intuitively, even though the static variables are held fixed, their residual constraints still exert an influence, which propagates to the dynamic variables through this coupling term.

By caching the factorization and solver associated with the static block $\mathbf{R}_{\mathcal{U}\mathcal{U}}$ (i.e., intermediate Schur complements), redundant computations are avoided across successive GN iterations, yielding substantial efficiency gains.
This strategy is conceptually analogous to restricting computation to the Markov blanket\footnote{In a pose graph, the Markov blanket of a node comprises its directly connected neighbors. Conditioning on this set renders the node independent of all others. Similarly, in SLAM, the effect of new measurements propagates primarily through the immediate neighbors of the affected variables. The cached block–Schur complement method leverages this principle: only variables reachable from the affected set via the elimination tree are re-solved, while the rest are left untouched.}, thereby localizing updates to the region of the pose graph most impacted by the new measurements.
Consequently, the partial solver computes the increments $\mathbf{d}_{\mathcal{S}_t}$ efficiently by reusing cached static-block solutions and solving only the reduced dynamic block through the Schur complement. This significantly lowers computational cost and enables scalable incremental SLAM.

\subsection{Determining Affected Variables}\label{find_affected}  

After each partial solve, we determine which variables still require updates and thus remain relevant for the next GN iteration. This amounts to updating the active set $\mathcal{S}_t$ through two complementary steps: pruning converged variables and expanding to maintain consistency.

\emph{Pruning}:  
Variables whose increments are sufficiently small are considered converged and are removed from the active set. Formally, we apply the threshold $\tau_d$ to the increment vector $\mathbf{d}_{\mathcal{S}_t}$, retaining only those indices with significant updates, i.e.,
\begin{equation}
\mathcal{S}_t \;\leftarrow\; \big\{\, i \in \mathcal{S}_t : |d_{t,i}| > \tau_d \,\big\},
\end{equation}
where $d_{t,i}$ denotes the $i$-th entry of $\mathbf{d}_{\mathcal{S}_t}$.  
Since poses and landmarks are represented by blocks of variables, we apply the pruning conservatively at the block level: if any variable within a node is retained, we keep the entire block of that node.

\emph{Expansion}:
Adjustments to the variables in $\mathcal{S}_t$ can induce new inconsistencies in the variables of neighboring nodes through their shared measurements (edges). To capture this effect, we collect all edges incident to the nodes whose variables are in $\mathcal{S}_t$:
\begin{equation}
\mathcal{E}_t \;=\; \bigcup_{n:\,\mathrm{vars}(n)\subseteq \mathcal{S}_t} \mathrm{edges}(n),
\end{equation}
where $\mathrm{vars}(n)$ denotes the variables of node $n$ and $\mathrm{edges}(n)$ its incident edges. We then enlarge the active set to include the variables of all nodes participating in these edges:
\begin{equation}
\mathcal{S}_t \;\leftarrow\; \mathcal{S}_t \;\cup\; \bigcup_{e \in \mathcal{E}_t} \bigcup_{n \in \mathrm{ends}(e)} \mathrm{vars}(n).
\end{equation}

To build intuition for the expansion step, consider an updated variable $x_i \in \mathcal{S}_t$ that participates in a scalar measurement $f_{i,k}(x_i, x_k)$ with a neighbor $x_k \notin \mathcal{S}_t$. The residual for this edge is
\begin{equation}
r_{ik} = f_{i,k}(x_i, x_k),
\end{equation}
with Jacobians $J_i = \tfrac{\partial f_{i,k}}{\partial x_i}$ and $J_k = \tfrac{\partial f_{i,k}}{\partial x_k}$.  
If $x_i$ is updated by an increment $d_i$, the residual changes approximately as
\begin{equation}
\Delta r_{i,k} \;\approx\; J_i d_i.
\end{equation}
For the edge to remain consistent, $x_k$ would need to be adjusted by an increment $d_k$ such that
\begin{equation}
J_i d_i + J_k d_k \;\approx\; 0.
\end{equation}
Hence, if $|J_i d_i|$ exceeds a threshold, the neighbor $x_k$ must also be included in $\mathcal{S}_t$. This illustrates how updates propagate through measurement connections, motivating the expansion rule.

In summary, the updated active set after each partial solve consists of 1) variables whose increments remain significant, 2) their full node blocks, and 3) all neighboring variables connected through incident edges. This prune-expand cycle ensures that only genuinely affected parts of the graph are revisited in the next iteration, balancing accuracy and efficiency.

\subsection{Selective Partial Optimization}

Once the active set of variables $\mathcal{S}_t$ has been initialized at increment $t$ from the new measurements (and possibly expanded via IGG, cf.\ Section~\ref{sec:gating}), GN iterations are carried out selectively until either $\mathcal{S}_t$ becomes empty or the maximum number of iterations $\tau_{\mathrm{GN}}$ is reached. Each iteration proceeds as follows:

\begin{enumerate}
\item \emph{Partial solve}: compute the GN step for the active variables by solving $(\mathbf{R}_t^\mathsf{T}\mathbf{R}_t) \mathbf{d} = \mathbf{b}_t$ restricted to $\mathbf{d}_{\mathcal{S}_t}$, i.e., the increments for the current active set, as described in section~\ref{sec:partial_solve}.
\item \emph{Active-set update}: prune converged variables (those with $|d_{t,i}| \le \tau_d$) and expand to include additional affected variables, preserving block structure and respecting pose-graph connectivity, to obtain the updated $\mathcal{S}_t$, as detailed in Section~\ref{find_affected}.
\item \emph{Convergence check}: terminate the GN iterations if $\mathcal{S}_t = \emptyset$.
\item \emph{State update}: update the current estimate for the active variables, $\mathbf{x}_{\mathcal{S}_t} \;\gets\; \mathbf{x}_{\mathcal{S}_t} - \mathbf{d}_{\mathcal{S}_t}$.
\item \emph{Relinearization}: recompute Jacobians and residuals only for edges involving variables in $\mathcal{S}_t$. Since all other variables remain fixed, this requires updating only the corresponding rows of $\mathbf{J}_t$ and entries of $\mathbf{R}_t$ and $\mathbf{b}_t$. This can be implemented as a lightweight refactorization of the affected portions of $\mathbf{R}_t$, e.g., via sparse low-rank factor update techniques or localized Bayes tree refactorization.
\end{enumerate}

This loop terminates when all increments fall below tolerance $\tau_d$ or $\tau_{\mathrm{GN}}$ iterations have been performed. The outcome is an updated state estimate $\mathbf{x}_t$ that approximately minimizes the nonlinear least-squares cost after including the new measurements at increment $t$. Crucially, we also maintain an updated factor $\mathbf{R}_t$, valid for the current linearization around $\mathbf{x}_t$, which is reused in subsequent increments, thus avoiding costly refactorization from scratch.

To build intuition, consider two limiting scenarios:

\begin{itemize}
\item \emph{Highly informative increments:} If each new measurement provides strong information (e.g., frequent loop closures), the information gain $\Delta \eta_t$ will exceed the threshold $\tau_\eta$, and $\mathcal{S}_t$ will initially include all variables. If all entries of $\mathbf{d}_t$ are as large as $\tau_d$, this effectively becomes batch optimization at each increment, ensuring maximum accuracy but at high computational cost.
\item \emph{Weakly informative increments:} If new measurements add little information (e.g., small odometry steps or redundant observations), only a small subset of recent variables enters $\mathcal{S}_t$. Updates are then highly localized, with minimal computational effort.
\end{itemize}

In practice, the behavior lies between these extremes. Odometry typically yields small or negligible updates, while loop closures or globally informative measurements trigger broader updates. The thresholds $\tau_\eta$ and $\tau_d$ play a stabilizing role: small improvements accumulate until $\tau_\eta$ is exceeded, at which point more variables are included and multiple GN iterations propagate corrections globally. This amortizes computation across increments and yields scalable performance in incremental SLAM.

\subsection{Computational Complexity}\label{complexity}

Each increment in our approach consists of up to $\tau_{\mathrm{GN}}$ GN iterations, each involving: (i) relinearization of a subset of variables, (ii) incremental (rank) Cholesky updates or downdates to the factor $\mathbf{R}_t$, and (iii) a partial solution of the underlying system of linear equations restricted to the affected variables. By keeping $\tau_{\mathrm{GN}}$ small (e.g., 5-10), we bound the per-increment cost similarly to iSAM2’s fluid relinearization scheme~\cite{kaess2012isam2}, but with an important distinction: computations are focused strictly on the active subset $\mathcal{S}_t$, which is typically a small fraction of the full state. This yields substantial savings in large graphs, where loop closures are infrequent or mostly local.

The main computational effort arises from two operations:
\begin{itemize}
\item \emph{Cholesky up(down)date}: When a variable is relinearized, the corresponding columns of $\mathbf{R}_t$ are modified. The cost of these operations is approximately
\begin{equation}
\min\left(2\sum_{i \in \mathcal{S}_t} \kappa_{t,i}^2,\; \sum_{i=1}^{N} \kappa_{t,i}^2\right),
\end{equation}
where $\kappa_{t,i}$ denotes the number of nonzeros in column $i$ of $\mathbf{R}_t$.
This expression shows that the update cost is always bounded by that of a full refactorization. For new edges, the cost is halved as no downdate is required.
\item \emph{Partial solve}: Once the system has been updated, the increment for the affected variables is computed via sparse triangular forward and backward substitutions. The cost is roughly
\begin{equation}
2\sum_{i \in \mathcal{S}_t} \kappa_{it},
\end{equation}
which scales linearly with the number of nonzeros in the relevant columns of $\mathbf{R}_t$. Equivalently, this corresponds to solving a small linear system restricted to the Markov blanket of $\mathcal{S}_t$, whose size is typically modest.
\end{itemize}

Because $\mathbf{R}_t$ is sparse and exhibits low fill-in under good orderings (e.g., planar or chain-like SLAM graphs), these operations are highly efficient in practice. Other operations, Jacobian construction, symbolic updates (e.g., reordering, elimination tree maintenance), and evaluation of information gain, incur negligible overhead compared to matrix factorization and solves.

Over $T$ increments, the total cost is driven by the cumulative size of the affected subsets $\mathcal{S}_t$. In the common case where loop closures or new observations only modify a bounded region of the graph, the accumulated cost scales linearly with $T$. In contrast, worst-case behavior, such as every increment triggering a global update, results in a cost proportional to $T$ full solves, though this is rare in real-world SLAM scenarios.

Hence, the overall cost depends on the total number of affected variables across all increments, rather than the total number of variables. In realistic scenarios, most updates remain local, and only a small fraction of increments trigger global corrections. This yields near-linear cumulative complexity in $T$.

In our implementation, we set $\tau_{\mathrm{GN}}=10$). Larger values yield diminishing returns within a single increment, since very large loop closures are uncommon in well-designed SLAM trajectories. If $\tau_{\mathrm{GN}}=1$, the method reduces to an iSAM2-style update (one linear solve per step, with occasional full relinearization, e.g., when $\tau_\eta$ is exceeded). At the other extreme, setting $\tau_\eta = 0$ (always update) and choosing large $\tau_{\mathrm{GN}}$ reproduces batch bundle adjustment at every step, which is accurate but prohibitively expensive. Our approach therefore spans the spectrum between naive incremental and full batch optimization, with thresholds $\tau_\eta$ and $\tau_d$ governing the trade-off between efficiency and accuracy.

As shown in Section~\ref{sec:experiments}, our information-guided gating allows most increments to proceed with localized updates only, while limiting the number of global updates to a small fraction of $T$. This preserves accuracy while substantially reducing runtime, leading to performance that is competitive with or better than existing incremental solvers, particularly on large graphs with frequent local updates.

\section{Theoretical Analysis}\label{sec:theory}

In this section, we show that the two heuristics at the core of our approach, (i) the information-guided gating and (ii) the selective partial optimization, do not compromise convergence. Under standard smoothness assumptions, they attain the same stationary point and local convergence rate as a full GN solver, while enabling the substantial FLOP savings reported in Section~\ref{sec:experiments}.

Recall the cost function~\eqref{cost} and let \(\mathbf{g}(\mathbf{x}) = \nabla c(\mathbf{x})\). We assume that the residual functions $\mathbf{r}_j$ in~\eqref{cost} are twice differentiable and that the GN Hessian $\nabla^2 c(\mathbf{x})$ is Lipschitz-continuous in a neighborhood of the solution $\mathbf{x}^\star$, i.e.,
\begin{equation}
\bigl\| \nabla^{2} c(\mathbf{x}) - \nabla^{2} c(\mathbf{y}) \bigr\|\le L\,\bigl\| \mathbf{x} - \mathbf{y} \bigr\|,\quad \forall\,\mathbf{x},\mathbf{y}\text{ near } \mathbf{x}^{\star},
\end{equation}
for some constant \(L>0\).

\subsection{IGG with Localized Updates}\label{sec:theory:trigger}

Let \(\mathcal{S}'\) denote the set of variables involved in the newly arrived measurements at increment~\(t\). If the information gain $\Delta \eta_t$ falls below the threshold $\tau_\eta$, the algorithm does not trigger a full graph update. Instead, it restricts computation to $\mathcal{S}'$ and solves the corresponding GN sub-system. The resulting step is therefore an inexact GN update whose residual satisfies
\begin{equation}
\bigl\|\mathbf{g}_{t} - \nabla^2_{\mathcal{S}'}c(\mathbf{x}_{t-1})\, \mathbf{d}_{\mathcal{S}'}\bigr\|\le\xi_t \, \|\mathbf{g}_{t-1}\| ,\quad \xi_t = 1-\frac{\Delta \eta_t}{\tau_\eta}<1.
\end{equation}
Since $\xi_t < 1$ whenever any positive information is assimilated, the step satisfies the forcing‐term condition in the Eisenstat-Walker inexact Newton framework~\cite{eisenstat1996inexact}, yielding the following results:

\begin{lemma} \label{lem:igg}
Suppose $c(\mathbf{x})$ has a Lipschitz-continuous GN Hessian in a neighborhood of the solution $\mathbf{x}^\star$. Then, the updates generated under the IGG trigger, with localized solves on $\mathcal{S}'$ when $\Delta \eta_t < \tau_\eta$, produce a monotonically non-increasing cost sequence and converge to a stationary point of~$c(\mathbf{x})$.
\end{lemma}

\begin{proof}
If $\Delta \eta_t \ge \tau_\eta$, a standard GN step is taken, which is a descent direction for $c(\mathbf{x})$. 
If $\Delta \eta_t < \tau_\eta$, the algorithm updates only the block $\mathcal{S}'$, yielding the exact block Newton step $\mathbf{d}_{\mathcal{S}'}$. Since the block Hessian $\nabla^2_{\!\mathcal{S}'}c(\mathbf{x})$ is positive definite in the neighborhood, the quadratic model ensures
\begin{equation}
c(\mathbf{x}_{t-1}) - c(\mathbf{x}_{t})\ge\alpha\|\mathbf{d}_{\mathcal{S}'}\|^2
\end{equation}
for some constant $\alpha > 0$.  
Moreover, the forcing-term bound of the Eisenstat-Walker inexact Newton framework~\cite{eisenstat1996inexact} is satisfied with $\xi_t = 1-\Delta \eta_t/\tau_\eta < 1$. Therefore, the classical convergence theory for inexact Newton applies, implying global convergence to a stationary point.
\end{proof}

\begin{corollary}
\label{cor:rate}
Under the assumptions of Lemma~\ref{lem:igg}, the iterates generated by the IGG trigger converge to the same stationary point $\mathbf{x}^\star$ as a full GN solver. Moreover, in the neighborhood of $\mathbf{x}^\star$, the local convergence rate is identical to that of standard GN: linear in general, and superlinear (quadratic) when the residuals vanish at the solution.
\end{corollary}

\begin{proof}
The proof follows directly from classical inexact Newton theory~\cite{eisenstat1996inexact}. Since the forcing term satisfies $\xi_t < 1$, global convergence is ensured. In the local regime, the accuracy of the block solve ensures that the forcing term can be made arbitrarily small as $\Delta \eta_t \to \tau_\eta$, which recovers the standard GN rate. Therefore, the method achieves the same asymptotic convergence behavior as full GN: linear when residuals persist at the solution, and quadratic when $\mathbf{r}(\mathbf{x}^\star) \to 0$.
\end{proof}

\subsection{SPO as Block-Coordinate Descent}

When the IGG trigger fires, we form a subset $\mathcal{S}_{t} \subseteq \{1,\dots,N_t\}$ of variables whose corresponding gradient entries exceed the tolerance $\tau_d$, and solve the GN subproblem restricted to $\mathcal{S}_{t}$. This procedure is equivalent to a greedy block-coordinate GN step. 
For smooth functions with positive-definite GN blocks, such steps guarantee a sufficient decrease of the form
\begin{equation}
c(\mathbf{x}_{t,l})\le c(\mathbf{x}_{t,l-1})-\alpha\|\nabla_{\mathcal{S}_t}c(\mathbf{x}_{t,l-1})\|^2
\end{equation}
for some constant $\alpha>0$. This ensures convergence provided that every variable enters the active set infinitely often~\cite{tseng2009coordinate,richtarik2016coordinate}. Here, $l$ denotes the GN iteration index within increment~$t$. To avoid clutter, we suppress this index and simply write $\mathbf{x}_t$, unless the inner iteration index is explicitly required. Our selection rule guarantees this condition by retaining each variable in $\mathcal{S}_t$ until its local gradient falls below $\tau_d$. Therefore, we have the following result:

\begin{lemma}\label{lem:spo}
Suppose each active block $\mathcal{S}_t$ is solved exactly, and that the residual gradient satisfies $\|\nabla_{\infty}c(\mathbf{x}_{t,l})\| \le \tau_d$ at the end of each increment. Then, the SPO scheme converges to the same stationary point $\mathbf{x}_t^\star$ as full GN, with local linear convergence. Moreover, if $\tau_d \to 0$ and $\mathbf{r}(\mathbf{x}^\star)\to 0$, the local rate is superlinear (quadratic in the ideal case), matching that of full GN.
\end{lemma}

\begin{proof}
Each SPO step is a greedy block Newton update on a smooth function with positive-definite block Hessian, and thus satisfies the sufficient decrease condition $c(\mathbf{x}_{t,l}) \le c(\mathbf{x}_{t,l-1}) - \alpha \|\nabla_{\mathcal{S}_t}c(\mathbf{x}_{t,l-1})\|^2$ \cite[Theorem~2.1]{tseng2009coordinate}. Because every variable whose gradient exceeds $\tau_d$ will eventually be added to $\mathcal{S}_t$, each coordinate is updated infinitely often or its gradient remains $\le \tau_d$. This guarantees convergence to a first-order stationary point. Standard GN theory~\cite[Ch.~6]{bjorck1996leastSquares} then implies local linear convergence, and superlinear (quadratic) convergence if $\tau_d \to 0$ and the residual vanishes at the solution.
\end{proof}

\subsection{Combined Convergence Guarantee}

Combining Lemmas~\ref{lem:igg} and~\ref{lem:spo} yields:

\begin{theorem} \label{thm:combined}
Assume $c(\mathbf{x})$ satisfies the standard smoothness conditions for local GN convergence. Then, the proposed approach combining the IGG scheme with threshold $\tau_\eta>0$ and the SPO strategy with tolerance $\tau_d>0$, generates a sequence $\mathbf{x}_t$ that (i) yields a monotonically non-increasing cost $c(\mathbf{x}_t)$, (ii) remains bounded, and (iii) converges to the same stationary point $\mathbf{x}^\star$ as batch GN. Moreover, the local asymptotic rate matches that of full GN, while empirical FLOP counts (Section~\ref{sec:experiments}) demonstrate a substantial reduction in numerical work.
\end{theorem}
\begin{corollary} \label{cor:efficiency}
The proposed approach is a computationally efficient realization of GN: it achieves the same stationary point $\mathbf{x}^\star$ and local convergence rate as batch GN, while asymptotically requiring fewer floating-point operations due to selective triggering and partial updates.
\end{corollary}

These theoretical results mirror our empirical findings: final accuracy is indistinguishable from full batch optimization, while computational cost is significantly reduced because (i) many low-information increments trigger only local updates, and (ii) later GN iterations act on progressively smaller active sets. A more refined analysis, such as incorporating adaptive forcing terms or stochastic update rules, remains an interesting direction for future work.

\section{Experiments}\label{sec:experiments}

\subsection{Setup}

\emph{Datasets}: We evaluate the proposed approach against the most closely related contenders on several standard SLAM benchmark datasets, summarized in Table~\ref{tab:datasets}. These datasets, provided in g2o~\cite{kummerle2011g2o} and TORO~\cite{grisetti2009nonlinear} formats and available from the online repository~\cite{carlone2015datasets}, cover scenarios with varying frequencies of loop closures and exhibit typical SLAM sparsity patterns~\cite{carlone2015initialization}.
To emulate realistic incremental operation, we reorder the edges in each dataset so that measurements arrive in the natural acquisition order, and loop-closure edges are incorporated as soon as both associated nodes have been initialized. This contrasts with the original datasets, where loop-closures are sometimes deferred until the end.
In addition, we construct a dataset, called MIT-P, by adding position priors to every $50$ poses in the MIT dataset, with Gaussian noise of standard deviation $1\,$m on each axis. These priors simulate scenarios where reasonably accurate external pose estimates are intermittently available, such as from indoor localization systems (e.g., UWB or WiFi-based) or from global navigation satellite systems (GNSS) outdoors. Since the considered datasets do not provide ground-truth poses, we use the batch solution as a surrogate ground truth for generating the position priors.

\begin{table}[t]
\small
\centering
\caption{Utilized benchmark SLAM datasets.}
\label{tab:datasets}
\begin{tabular}{lccccc}
\toprule
dataset & poses & edges & loop closures & $\tau_d$ & $\tau_\eta$\\
\midrule
MIT, MIT-P & 808   & 827   & 20   & $10^{-3}$ & 1   \\
FR079   & 989   & 1217  & 229  & $10^{-4}$ & 0.6 \\
CSAIL   & 1045  & 1172  & 128  & $10^{-5}$ & 0.95\\
Intel   & 1228  & 1483  & 256  & $10^{-6}$ & 0.72\\
FRH     & 1316  & 2820  & 1505 & $10^{-7}$ & 0.45\\
\bottomrule
\end{tabular}
\end{table}

\noindent\emph{Accuracy measures}: We use the normalized chi-squared error, denoted by $\mathrm{N}\chi^2_{t}$, and the absolute trajectory error (ATE) as our accuracy measures. For each approach, we report both their final values (after the last increment) and their average values (over all increments). The normalized chi-squared error is directly related to the nonlinear least-squares cost~\eqref{cost} as
\begin{equation}
\mathrm{N}\chi^2_{t} = \frac{2c(\mathbf{x}_t)}{M_t}
\end{equation}
where $M_t$ is the number of scalar measurement equations available at increment~$t$. 
The ATE quantifies the deviation of the estimated trajectory from ground truth. It is computed as the root mean squared error (RMSE) between the estimated poses $\mathbf{x}_{t,p}$ and the corresponding ground-truth poses $\mathbf{x}^\star_p$, after alignment by a rigid-body transformation, such as the Kabsch algorithm~\cite{kabsch1976solution,kabsch1978discussion}, i.e.,
\begin{equation}
\mathrm{ATE} = \sqrt{\frac{1}{P_t}\sum_{p=1}^{P_t}\bigl\|\mathbf{x}^\star_p - \mathbf{T}_t\, \mathbf{x}_{t,p}\bigr\|^2},
\end{equation}
where $P_t$ is the number of poses at increment $t$, and $\mathbf{T}_t\in\mathrm{SE}(2)$ denotes the optimal alignment transformation at that increment. As before, we take the batch solution as the reference ground truth.

\noindent\emph{Considered approaches}: We include the following approaches in our evaluations:

\begin{itemize}
\item {GN1}: Performs a \emph{single} GN iteration per increment. This setting resembles iSAM2~\cite{kaess2012isam2} without any relinearization threshold but with a variable-update threshold. For fairness, new measurements are relinearized exactly as in our approach, ensuring comparable treatment of updates.

\item {GNi}: Performs \emph{multiple} GN iterations at each increment without any selectivity, i.e., no gating or partial optimization. All variables may be updated and relinearized whenever new measurements arrive. To maintain consistency with other approaches, the threshold $\tau_d$ is used to decide early termination of GN iterations (cf. Algorithm~\ref{alg:gni}). 
This algorithm resembles a standard incremental solver without periodic batch steps. It is maximally responsive and accurate but may perform considerable redundant work when information gain is small.

\item {GNi-LCG}: Similar to GNi, but performs GN iterations \emph{only} when a loop closure is detected, hence the term loop-closure gating (LCG). Loop closures are identified as measurements that connect previously unconnected parts of the graph (e.g., a pose-pose constraint between non-consecutive poses). This strategy is comparable to approaches proposed in~\cite{polok2013incrementalChol,slampp2017}.

\item {GNi-IGG}:  Similar to GNi, but performs GN iterations \emph{only} when the information gain exceeds the threshold $\tau_\eta$, according to the proposed IGG scheme.

\item {GNi-SPO}: Extends GNi by incorporating the proposed SPO scheme. Multiple GN iterations may be performed at each increment, but updates are restricted to the active set of affected variables. No gating is applied, meaning all variables are considered potentially affected at every increment.

\item {GNi-SPO-LCG}: Combines SPO with LCG. As in GNi-SPO, only the active set of affected variables is updated and relinearized, while global GN optimization is triggered solely at loop closures. When loop closures are the primary source of information, this approach is expected to behave similarly to the proposed approach.

\item {GNi-SPO-IGG}: The proposed algorithm, capable of performing multiple GN iterations at each increment while incorporating both SPO and IGG.
\end{itemize}

\begin{algorithm}[!t]
\caption{Incremental SLAM with no Gating or Partial Optimization}
\label{alg:gni}
\small
\begin{algorithmic}[1]
\Require initial estimate $\mathbf{x}_0$, thresholds $\tau_d$ and $\tau_{\mathrm{GN}}$
\Ensure updated estimate $\mathbf{x}_t$ after each increment $t$
\For{$t = 1, 2, \dots$}
    \State update $\mathbf{J}_t$, $\mathbf{r}_t$, $\mathbf{R}_t$, $\mathbf{b}_t$, and $\mathbf{p}_t$ based on new measurements
    \For{$i_{\mathrm{GN}}=1,2,\dots,\tau_{\mathrm{GN}}$}
        \State solve $(\mathbf{R}_t^\mathsf{T}\mathbf{R}_t) \mathbf{d}_t = \mathbf{b}_t$ for $\mathbf{d}_t$
        \If{$\max|\mathbf{d}_t|\le\tau_d$}
            break
        \EndIf
        \State $\mathbf{x}_t \gets \mathbf{x}_t - \mathbf{d}_t$
        \State recalculate $\mathbf{J}_t$, $\mathbf{r}_t$, $\mathbf{R}_t$, $\mathbf{b}_t$, and $\mathbf{p}_t$ based on new $\mathbf{x}_t$
    \EndFor
\EndFor
\end{algorithmic}
\end{algorithm}

\noindent\emph{Parameters}: We set $\tau_{\mathrm{GN}}=10$ for all approaches, except for GN1, which corresponds to GNi with $\tau_{\mathrm{GN}}=1$. We set the thresholds $\tau_d$ and $\tau_\eta$ separately for each dataset, as listed in Table~\ref{tab:datasets}. We perform minimal tuning of these thresholds to ensure that the proposed approach attains accuracy comparable to its competitors, thereby allowing for fair efficiency comparisons. We consider each incremental step involve one measurement (edge).

\subsection{Results}

In Table~\ref{tab:results}, we summarizes the performance evaluation results for all considered approaches across the benchmark datasets. Alongside final and mean N$\chi^2$ and ATE values, which capture estimation accuracy, we also report average floating-point operation (FLOP) counts for Cholesky up(down)date (\emph{update}) and partial-solve (\emph{solve}) steps, as defined in Section~\ref{complexity}, providing a direct measure of computational efficiency over all increments. 
For GNi-SPO-IGG, GNi-SPO-LCG, and GNi-SPO, we additionally report (in parentheses) the mean solve FLOPs for when partial solves are replaced with full solves (see line~10 of Algorithm~\ref{alg:selective_update}), while still retaining partial variable update and relinearization. This variant isolates the benefit of the partial-solve strategy itself in terms of computational savings. Conceptually, it is analogous to iSAM2 with fluid relinearization when the same variable update and relinearization threshold $\tau_d$ is applied. Note that for this full-solve variant, accuracy and mean update FLOPs remain unchanged compared to the partial-solve case as only the solve FLOPs differ.

To evaluate both intermediate accuracy and computational progression, we plot the evolution of ATE and cumulative update FLOPs over all increments for all considered approaches on four representative datasets, as shown in Figs.~\ref{fig:ate} and~\ref{fig:flops}, respectively.
To provide further insight into the behavior of the proposed IGG scheme, in Fig.~\ref{fig:info}, we plot the information gain $\Delta \eta_t$ across all increments for four datasets. The plots also indicate increments at which loop closures are detected or position priors are introduced.

To examine the influence of the threshold values $\tau_d$ and $\tau_\eta$ on the accuracy-efficiency trade-off of the proposed approach, in Fig.~\ref{fig:ATE_d}, we plot the mean ATE and mean update FLOPs of GNi-SPO on the MIT dataset as functions of $\tau_d$, and, in Fig.~\ref{fig:ATE_eta}, those of GNi-IGG on the FR079 dataset as functions of $\tau_\eta$. For comparison with relevant baselines, Fig.~\ref{fig:ATE_d} also includes the corresponding curves for GNi and GN1, while Fig.~\ref{fig:ATE_eta} includes those for GNi-LCG, GNi, and GN1.

\begin{table*}
\small
\centering
\caption{Performance of different approaches on the considered datasets.}
\label{tab:results}
\mbox{}\clap{
\begin{tabular}{llcccccc}
\toprule
\multirow{2}{*}{dataset} & \multirow{2}{*}{approach} & \multicolumn{2}{c}{N$\chi^2$} & \multicolumn{2}{c}{ATE} & \multicolumn{2}{c}{mean FLOPs} \\
& & final & mean & final & mean & solve & update \\
\midrule
MIT    & GNi-SPO-IGG  & $1.65918{\times}10^{-2}$ & $1.84891{\times}10^{-2}$ & $3.67389{\times}10^{-4}$  & $5.802394$               & 2,028  (36,926) & 66,541    \\
       & GNi-SPO-LCG  & $1.65918{\times}10^{-2}$ & $1.84891{\times}10^{-2}$ & $3.67389{\times}10^{-4}$  & $5.802394$               & 2,028  (36,926) & 66,541    \\
       & GNi-SPO      & $1.65915{\times}10^{-2}$ & $1.84891{\times}10^{-2}$ & $3.47418{\times}10^{-4}$  & $5.802397$               & 19,036 (36,925) & 66,565    \\
       & GNi-IGG      & $32.5859$                & $14.6328$                & $30.4408$                 & $20.3804$                & 3,696   & 98,087    \\
       & GNi-LCG      & $32.5859$                & $14.6328$                & $30.4408$                 & $20.3804$                & 3,696   & 98,087    \\
       & GNi          & $1.65914{\times}10^{-2}$ & $1.84841{\times}10^{-2}$ & $5.20559{\times}10^{-11}$ & $5.802427$               & 36,661  & 438,548   \\
       & GN1          & $1.65914{\times}10^{-2}$ & $780,578$                & $1.05711{\times}10^{-10}$ & $5.850329$               & 17,704  & 405,989   \\
\midrule
FR079  & GNi-SPO-IGG  & $1.02983{\times}10^{-2}$ & $1.06656{\times}10^{-2}$ & $3.75248{\times}10^{-5}$  & $6.02654{\times}10^{-2}$ & 8,377 (50,546)  & 85,685    \\
       & GNi-SPO-LCG  & $1.02983{\times}10^{-2}$ & $1.06653{\times}10^{-2}$ & $3.75245{\times}10^{-5}$  & $6.02615{\times}10^{-2}$ & 7,355 (50,620)  & 87,015    \\
       & GNi-SPO      & $1.02983{\times}10^{-2}$ & $1.06653{\times}10^{-2}$ & $3.75247{\times}10^{-5}$  & $6.02615{\times}10^{-2}$ & 28,318 (50,620) & 87,034    \\
       & GNi-IGG      & $5.84489{\times}10^{-2}$ & $5.36751{\times}10^{-1}$ & $3.89112{\times}10^{-2}$  & $8.87086{\times}10^{-2}$ & 12,085  & 130,969   \\
       & GNi-LCG      & $5.84489{\times}10^{-2}$ & $1.81984{\times}10^{-1}$ & $3.89112{\times}10^{-2}$  & $1.22274{\times}10^{-1}$ & 9,334   & 101,979   \\
       & GNi          & $1.02983{\times}10^{-2}$ & $1.06651{\times}10^{-2}$ & $9.95967{\times}10^{-15}$ & $6.02609{\times}10^{-2}$ & 50,620  & 460,584   \\
       & GN1          & $1.02983{\times}10^{-2}$ & $1.06652{\times}10^{-2}$ & $2.78002{\times}10^{-12}$ & $6.02635{\times}10^{-2}$ & 24,808  & 437,461   \\
\midrule
CSAIL  & GNi-SPO-IGG  & $1.10797{\times}10^{-2}$ & $2.80792{\times}10^{-3}$ & $1.23096{\times}10^{-6}$  & $9.11474{\times}10^{-2}$ & 10,245 (51,960) & 268,636   \\
       & GNi-SPO-LCG  & $1.10797{\times}10^{-2}$ & $2.80776{\times}10^{-3}$ & $1.23113{\times}10^{-6}$  & $9.11517{\times}10^{-2}$ & 10,147 (52,127) & 270,687   \\
       & GNi-SPO      & $1.10797{\times}10^{-2}$ & $2.80776{\times}10^{-3}$ & $1.23121{\times}10^{-6}$  & $9.11517{\times}10^{-2}$ & 29,651 (52,128) & 271,536   \\
       & GNi-IGG      & $1.10797{\times}10^{-2}$ & $2.34198$                & $4.13827{\times}10^{-14}$ & $2.25203{\times}10^{-1}$ & 14,657  & 386,088   \\
       & GNi-LCG      & $1.10797{\times}10^{-2}$ & $3.8143$                 & $5.18019{\times}10^{-7}$  & $2.25553{\times}10^{-1}$ & 15,134  & 401,049   \\
       & GNi          & $1.10797{\times}10^{-2}$ & $2.80718{\times}10^{-3}$ & $2.29725{\times}10^{-14}$ & $9.11515{\times}10^{-2}$ & 52,053  & 978,461   \\
       & GN1          & $1.10797{\times}10^{-2}$ & $2.80897{\times}10^{-3}$ & $1.35270{\times}10^{-6}$  & $9.11548{\times}10^{-2}$ & 23,974  & 825,271   \\
\midrule
Intel  & GNi-SPO-IGG  & $4.85217{\times}10^{-2}$ & $3.42609{\times}10^{-2}$ & $1.01812{\times}10^{-7}$  & $1.40955{\times}10^{-1}$ & 28,609 (77,951) & 330,270   \\
       & GNi-SPO-LCG  & $4.85121{\times}10^{-2}$ & $3.42331{\times}10^{-2}$ & $4.06794{\times}10^{-7}$  & $1.40951{\times}10^{-1}$ & 20,362 (78,423) & 286,034   \\
       & GNi-SPO      & $4.85121{\times}10^{-2}$ & $3.42397{\times}10^{-2}$ & $1.18840{\times}10^{-7}$  & $1.40951{\times}10^{-1}$ & 49,525 (78,686) & 357,216   \\
       & GNi-IGG      & $69.5308$                & $126.965$                & $1.80246$                 & $5.97116{\times}10^{-1}$ & 39,008  & 435,542   \\
       & GNi-LCG      & $6.01180$                & $19,366.4$               & $2.86431$                 & $1.27203$                & 26,905  & 336,683   \\
       & GNi          & $4.85121{\times}10^{-2}$ & $3.42216{\times}10^{-2}$ & $2.18560{\times}10^{-11}$ & $1.40951{\times}10^{-1}$ & 77,391  & 709,119   \\
       & GN1          & $4.85121{\times}10^{-2}$ & $1.06152$                & $1.72933{\times}10^{-11}$ & $1.40979{\times}10^{-1}$ & 31,523  & 488,865   \\
\midrule
FRH    & GNi-SPO-IGG  & $2.28294{\times}10^{-8}$ & $1.11147{\times}10^{-8}$ & $8.95582{\times}10^{-11}$ & $3.03247{\times}10^{-4}$ & 34,307 (98,596) & 1,308,496 \\
       & GNi-SPO-LCG  & $2.28294{\times}10^{-8}$ & $1.11142{\times}10^{-8}$ & $8.95220{\times}10^{-11}$ & $3.03360{\times}10^{-4}$ & 35,814 (99,467) & 1,465,775 \\
       & GNi-SPO      & $2.28294{\times}10^{-8}$ & $1.11142{\times}10^{-8}$ & $8.95004{\times}10^{-11}$ & $3.03360{\times}10^{-4}$ & 57,870 (99,467) & 1,469,435 \\
       & GNi-IGG      & $2.28294{\times}10^{-8}$ & $5.90106{\times}10^{-7}$ & $1.94511{\times}10^{-13}$ & $3.46361{\times}10^{-4}$ & 55,132  & 3,614,516 \\
       & GNi-LCG      & $6.16115{\times}10^{-3}$ & $2.82510{\times}10^{-6}$ & $1.66055{\times}10^{-4}$  & $3.46707{\times}10^{-4}$ & 56,504  & 3,742,194 \\
       & GNi          & $2.28294{\times}10^{-8}$ & $1.11140{\times}10^{-8}$ & $8.60015{\times}10^{-14}$ & $3.03360{\times}10^{-4}$ & 99,467  & 6,141,656 \\
       & GN1          & $2.28294{\times}10^{-8}$ & $1.11140{\times}10^{-8}$ & $1.51830{\times}10^{-13}$ & $3.03360{\times}10^{-4}$ & 49,737  & 6,135,811 \\
\midrule
MIT-P  & GNi-SPO-IGG  & $1.72741{\times}10^{-2}$ & $2.06264{\times}10^{-2}$ & $1.73090{\times}10^{-3}$  & $1.384435$               & 3,799  (40,348) & 100,436   \\
       & GNi-SPO-LCG  & $1.79920{\times}10^{-2}$ & $9.21434{\times}10^{-1}$ & $5.77128{\times}10^{-2}$  & $2.324933$               & 2,650  (40,664) & 77,984    \\
       & GNi-SPO      & $1.72743{\times}10^{-2}$ & $1.99909{\times}10^{-2}$ & $1.64191{\times}10^{-3}$  & $2.101596$               & 25,159 (56,904) & 187,886   \\
       & GNi-IGG      & $71.4849$                & $40.6574$                & $25.0316$                 & $24.25101$               & 7,762   & 194,298   \\
       & GNi-LCG      & $93.2878$                & $45.3580$                & $36.7726$                 & $31.01479$               & 4,239   & 115,386   \\
       & GNi          & $1.72738{\times}10^{-2}$ & $1.99852{\times}10^{-2}$ & $3.85749{\times}10^{-12}$ & $2.101118$               & 56,753  & 803,583   \\
       & GN1          & $1.72738{\times}10^{-2}$ & $1.50452$                & $9.22196{\times}10^{-6}$  & $1.537856$               & 17,722  & 406,185   \\
\bottomrule
\end{tabular}}
\end{table*}

\begin{figure*}
\centering
\begin{subfigure}{0.5\linewidth}
    \includegraphics[scale=.27]{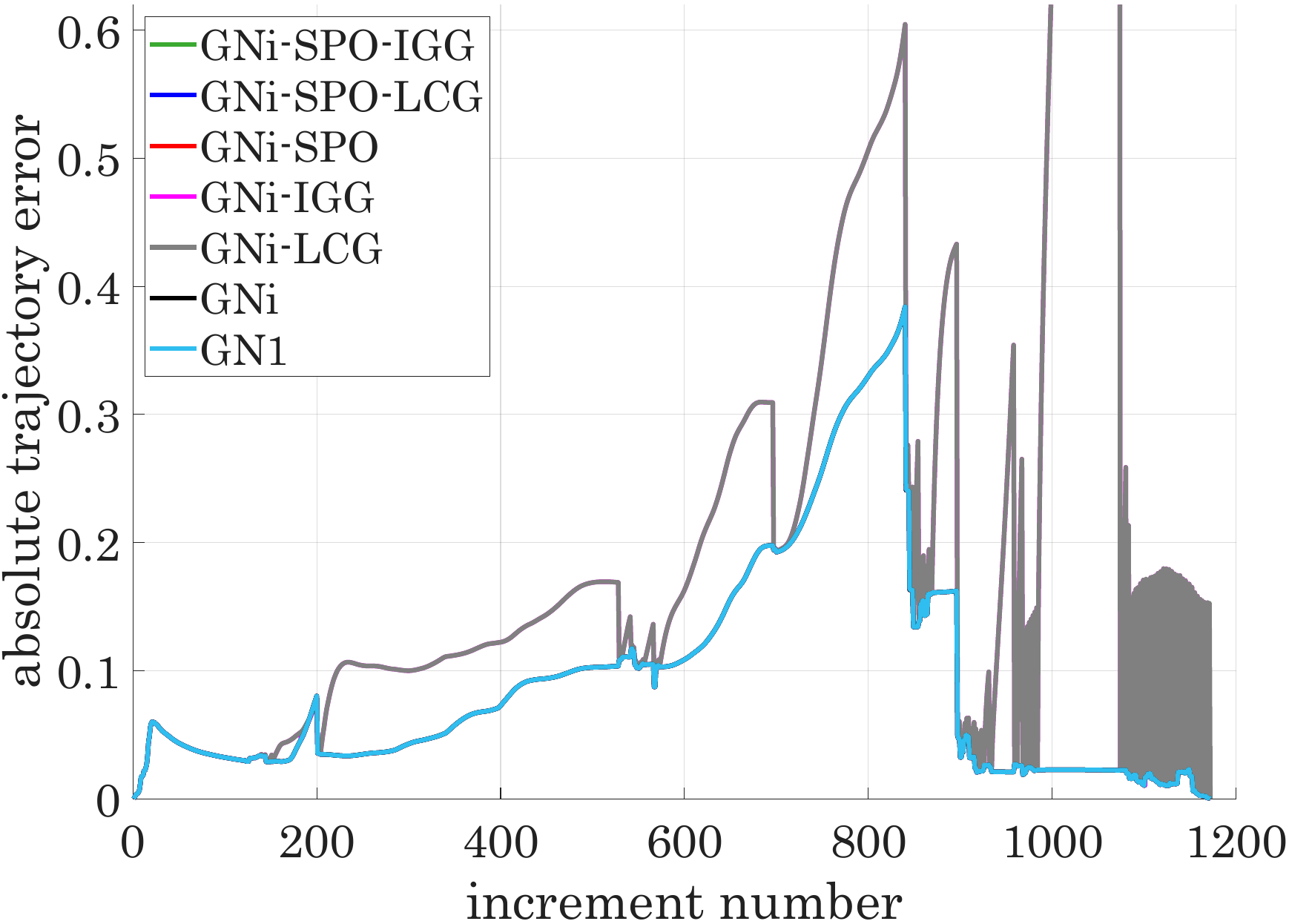}
    \caption{CSAIL}
    \label{fig:ate_csail}
\end{subfigure}\hfill
\begin{subfigure}{0.5\linewidth}
    \includegraphics[scale=.27]{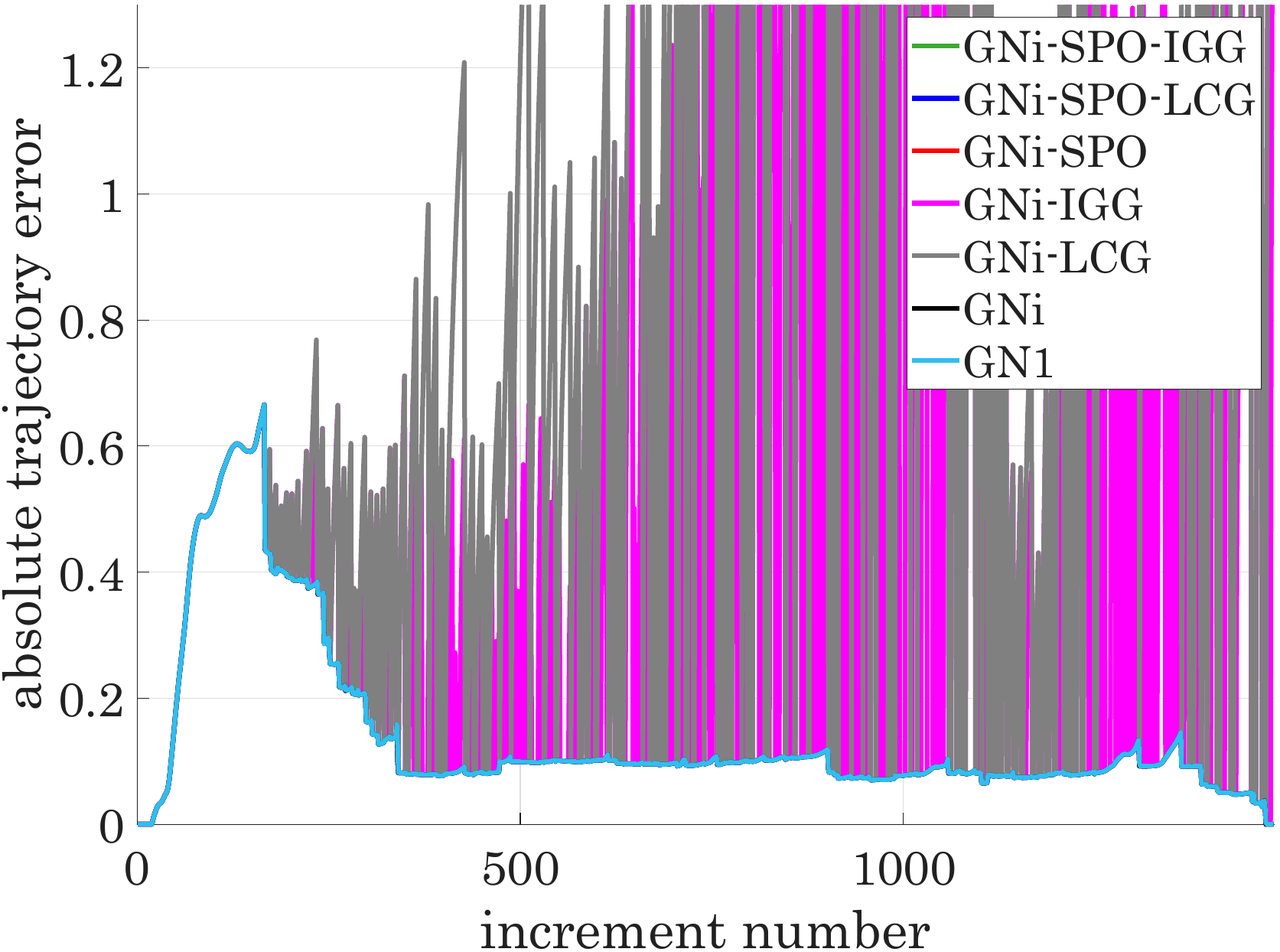}
    \caption{Intel}
    \label{fig:ate_intel}
\end{subfigure}
\begin{subfigure}{0.5\linewidth}
    \includegraphics[scale=.27]{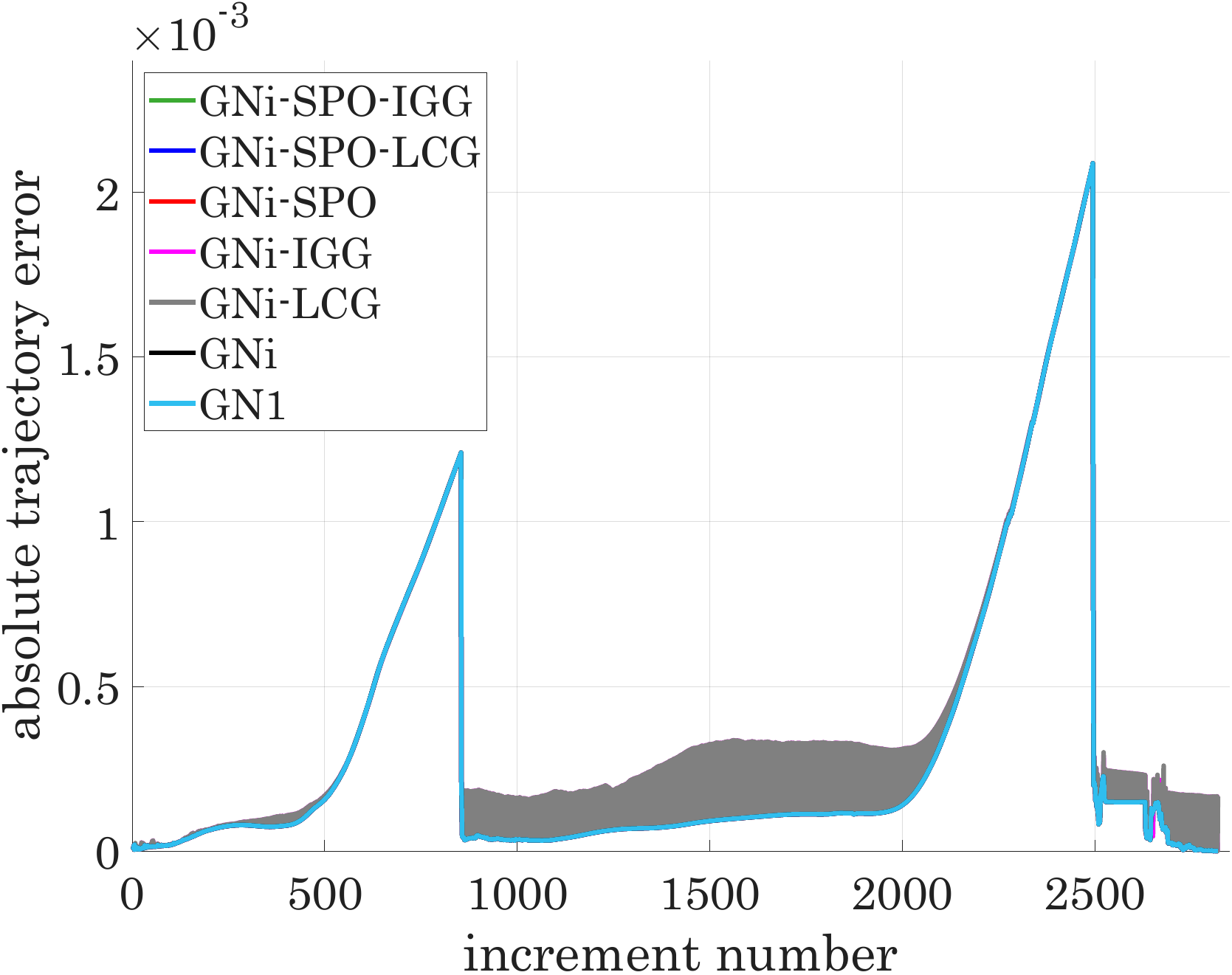}
    \caption{FRH}
    \label{fig:ate_frh}
\end{subfigure}\hfill
\begin{subfigure}{0.5\linewidth}
    \includegraphics[scale=.27]{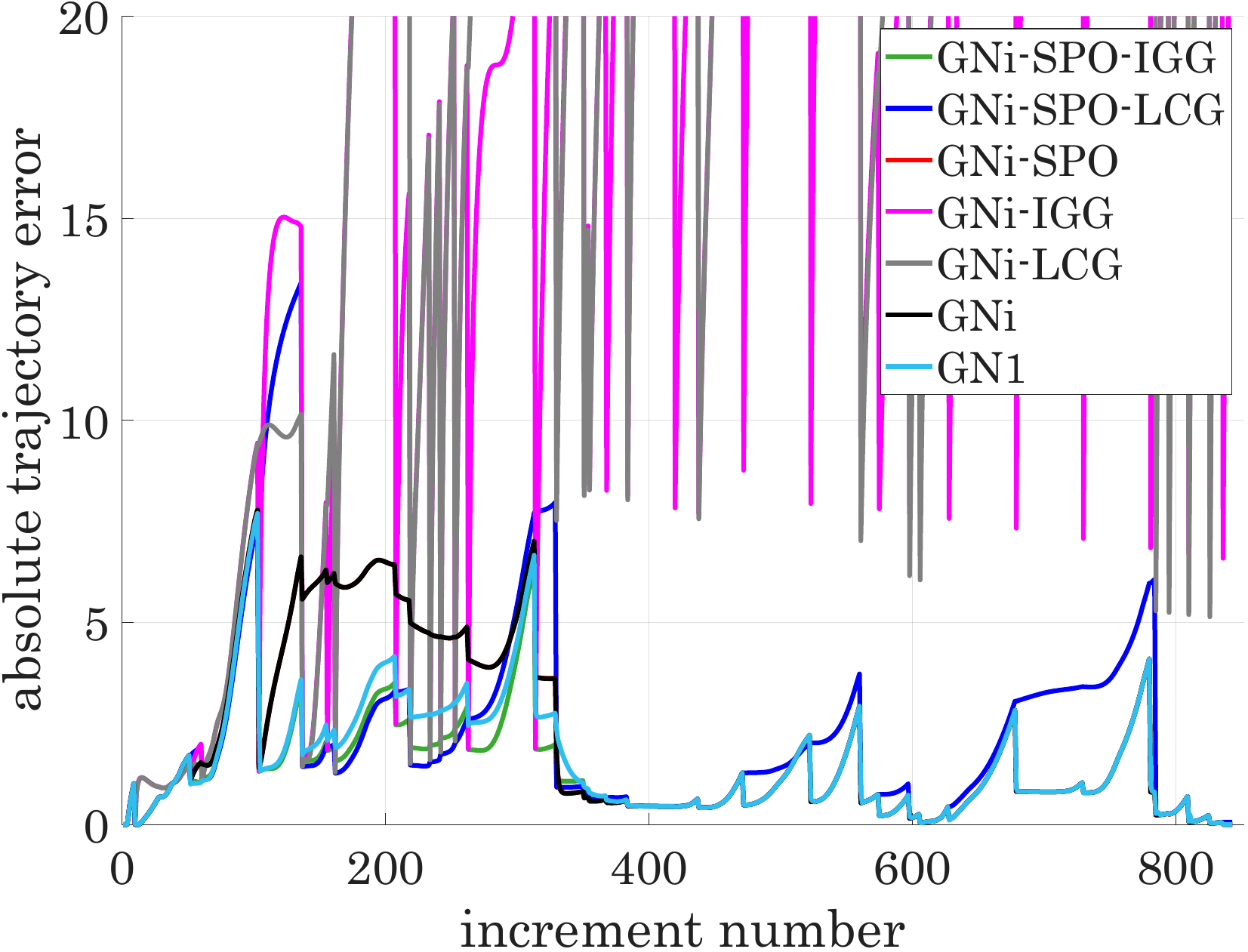}
    \caption{MIT-P}
    \label{fig:ate_mitp}
\end{subfigure}
\caption{ATE over increments for four datasets.}
\label{fig:ate}
\end{figure*}

\begin{figure*}
\centering
\begin{subfigure}{0.5\linewidth}
    \includegraphics[scale=.27]{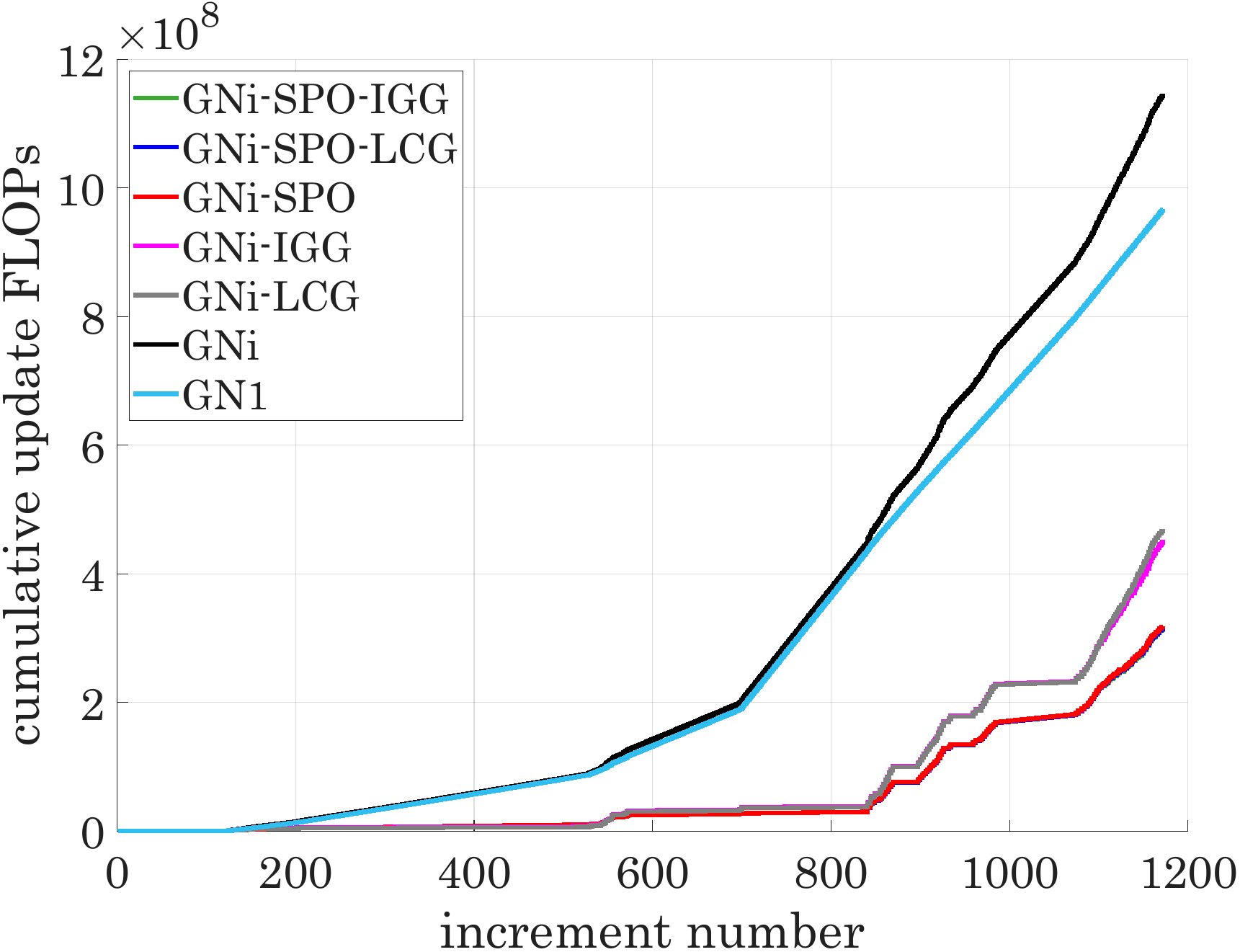}
    \caption{CSAIL}
    \label{fig:update_csail}
\end{subfigure}\hfill
\begin{subfigure}{0.5\linewidth}
    \includegraphics[scale=.27]{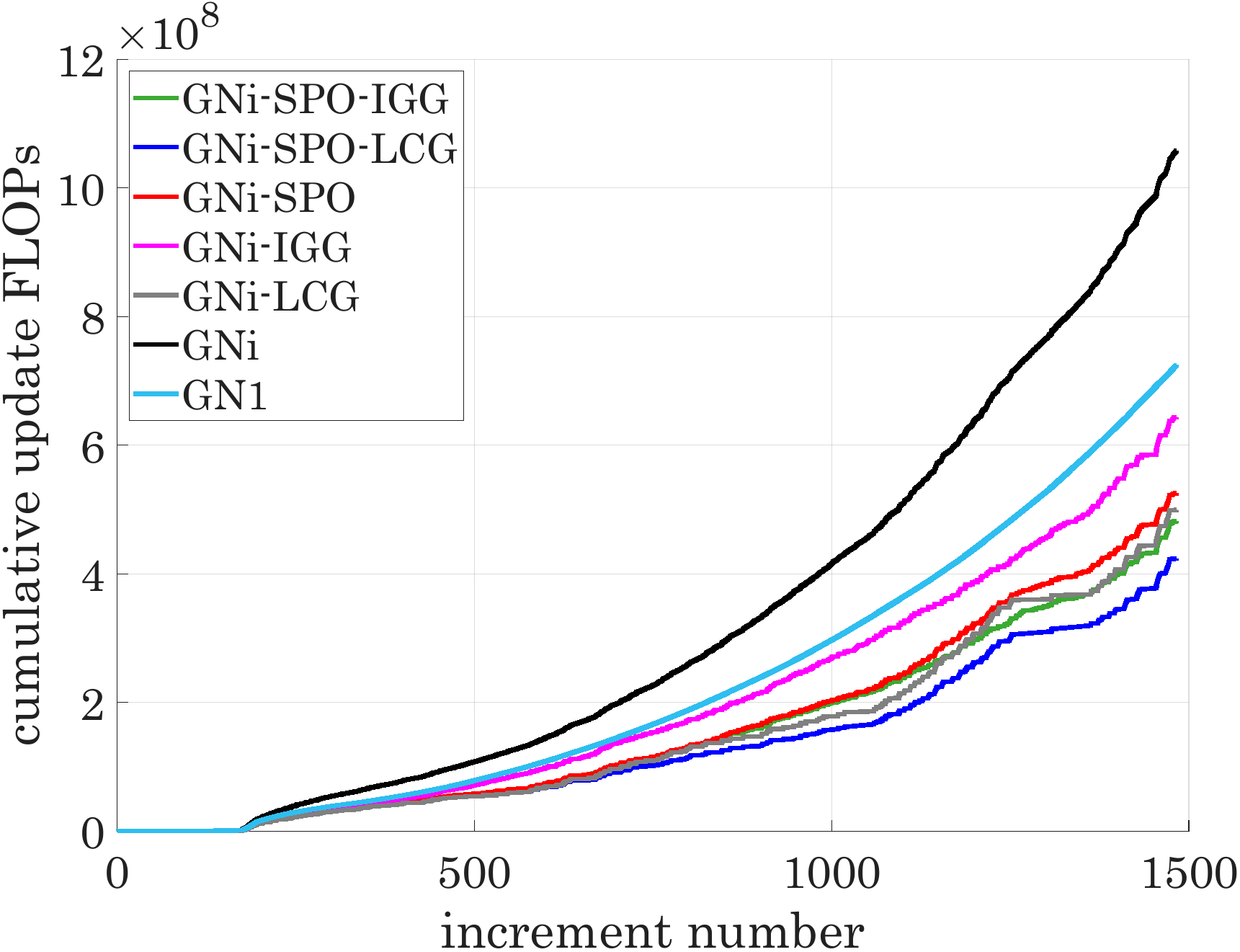}
    \caption{Intel}
    \label{fig:update_intel}
\end{subfigure}
\begin{subfigure}{0.5\linewidth}
    \includegraphics[scale=.27]{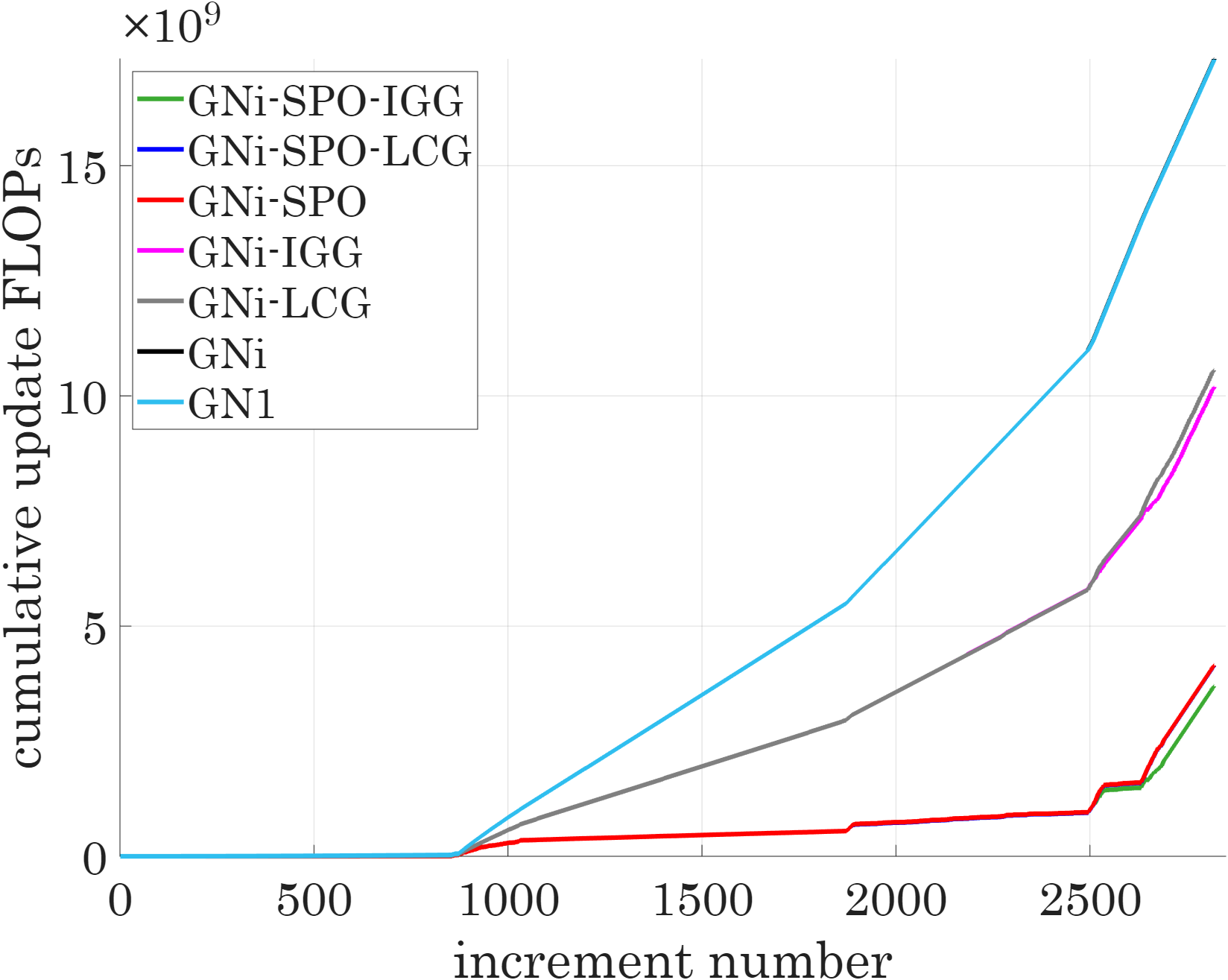}
    \caption{FRH}
    \label{fig:update_frh}
\end{subfigure}\hfill
\begin{subfigure}{0.5\linewidth}
    \includegraphics[scale=.27]{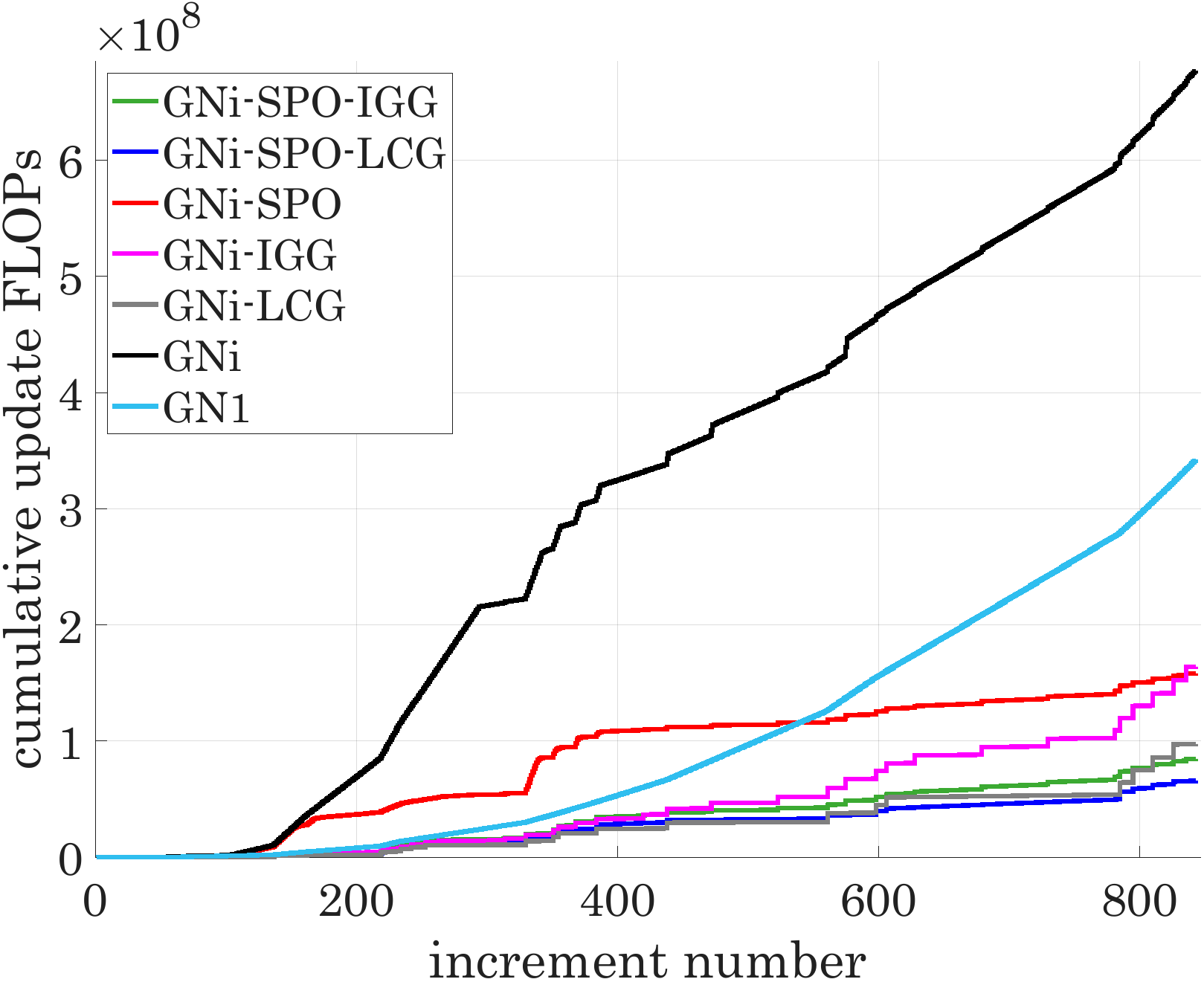}
    \caption{MIT-P}
    \label{fig:update_mitp}
\end{subfigure}
\caption{Cumulative update FLOPs over increments for four dataset.}
\label{fig:flops}
\end{figure*}

\begin{figure*}
\centering
\begin{subfigure}{0.5\linewidth}
    \includegraphics[scale=.27]{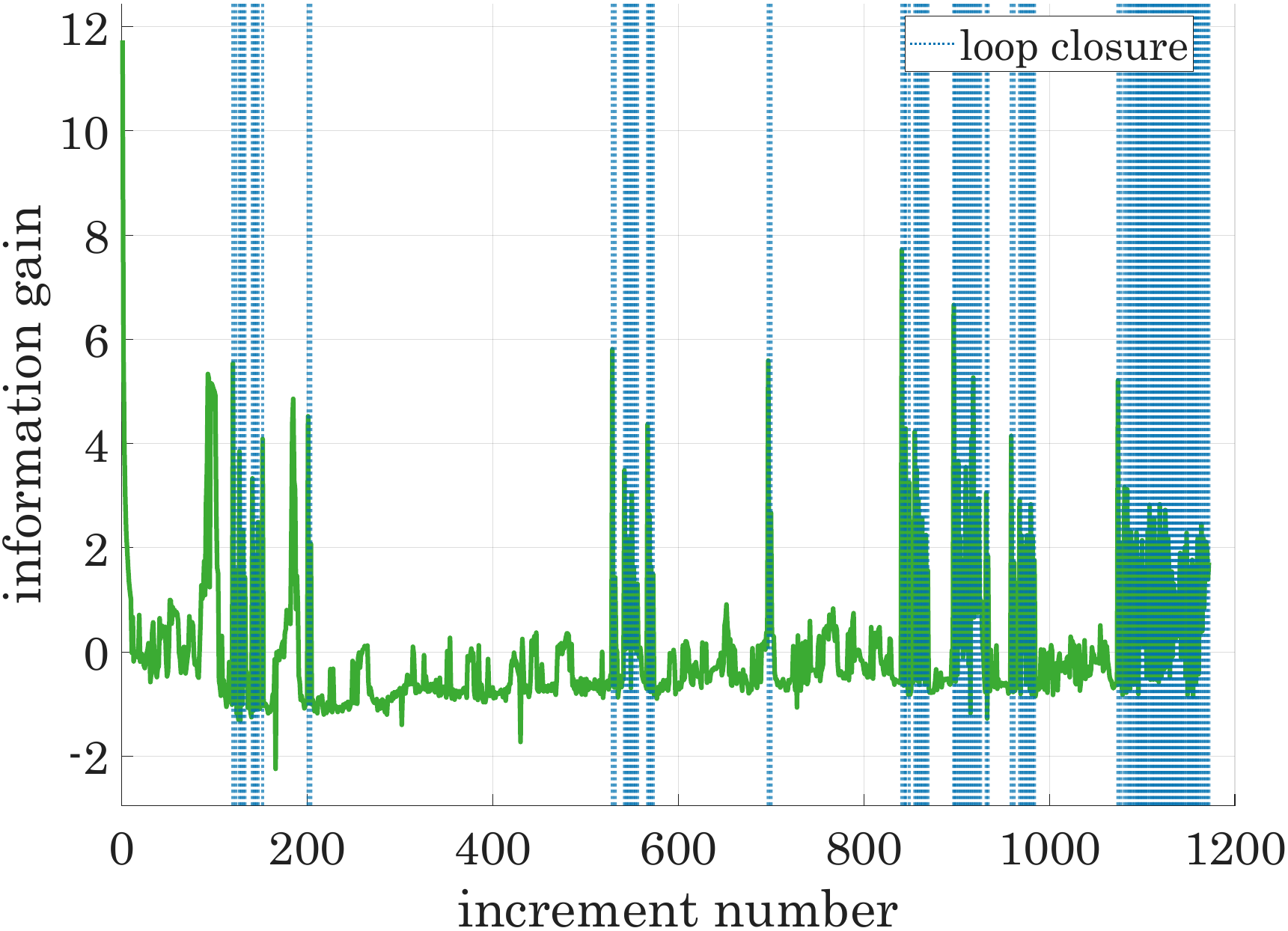}
    \caption{CSAIL}
    \label{fig:info_CSAIL}
\end{subfigure}\hfill
\begin{subfigure}{0.5\linewidth}
    \includegraphics[scale=.27]{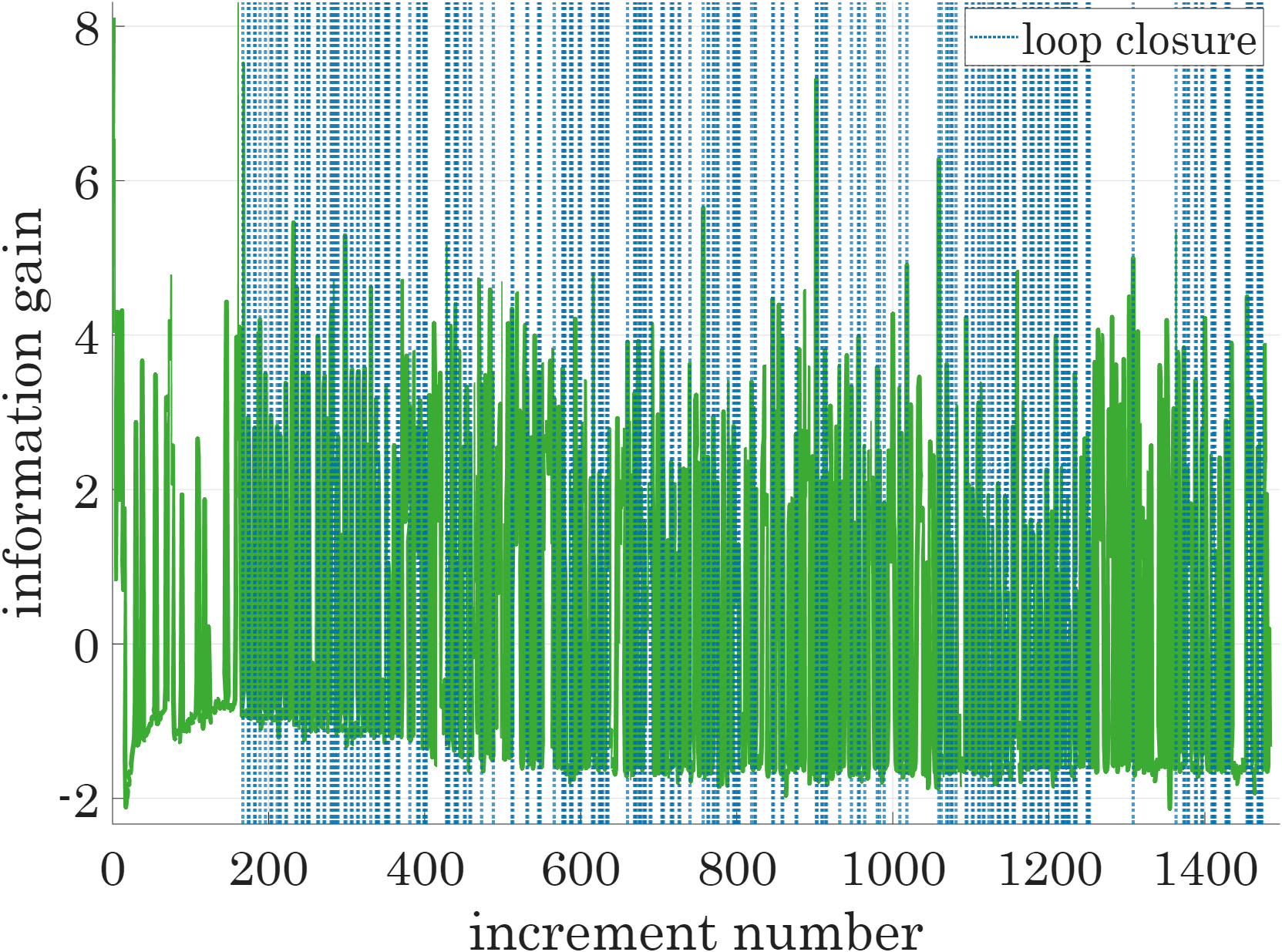}
    \caption{Intel}
    \label{fig:info_intel}
\end{subfigure}
\begin{subfigure}{0.5\linewidth}
    \includegraphics[scale=.27]{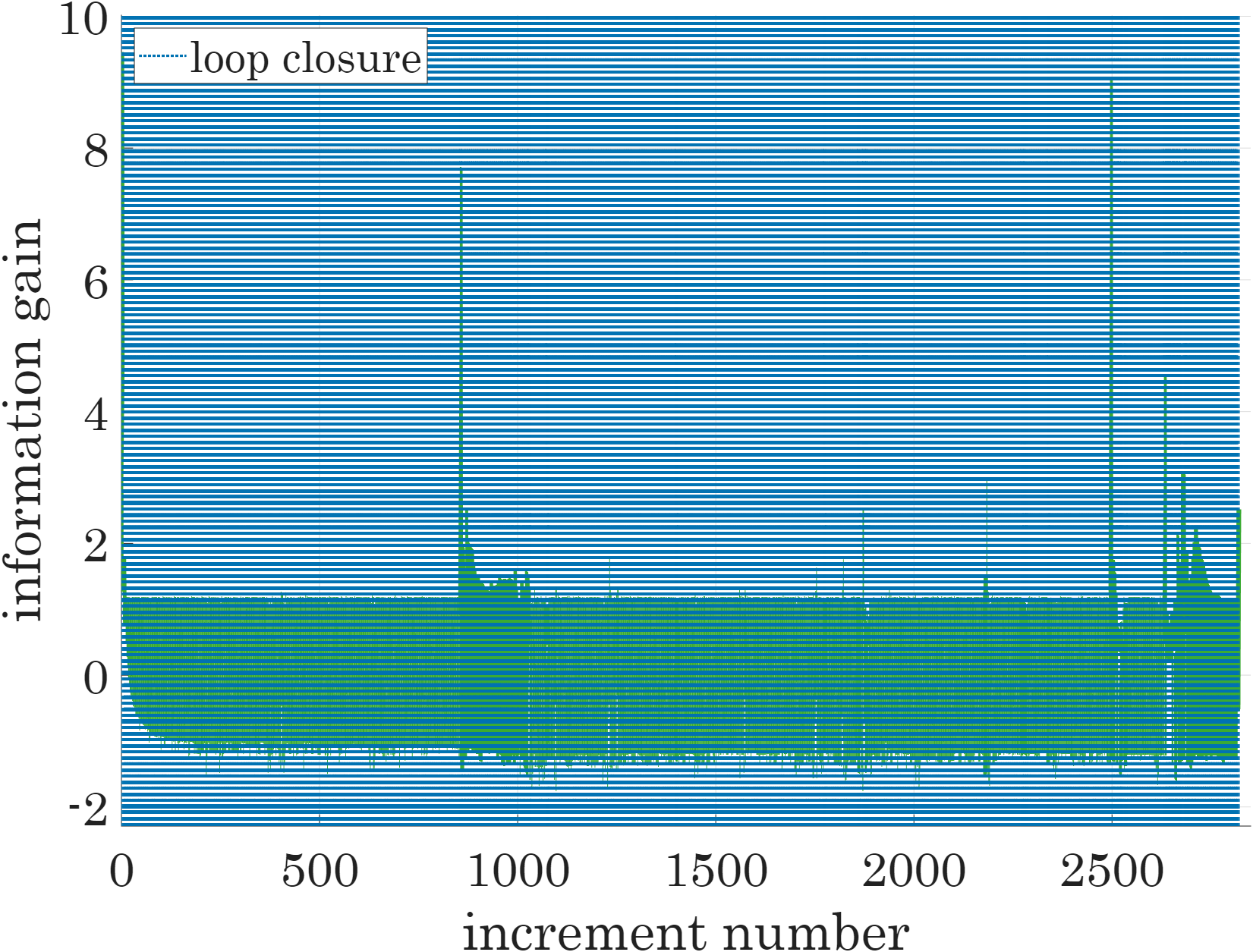}
    \caption{FRH}
    \label{fig:info_frh}
\end{subfigure}\hfill
\begin{subfigure}{0.5\linewidth}
    \includegraphics[scale=.27]{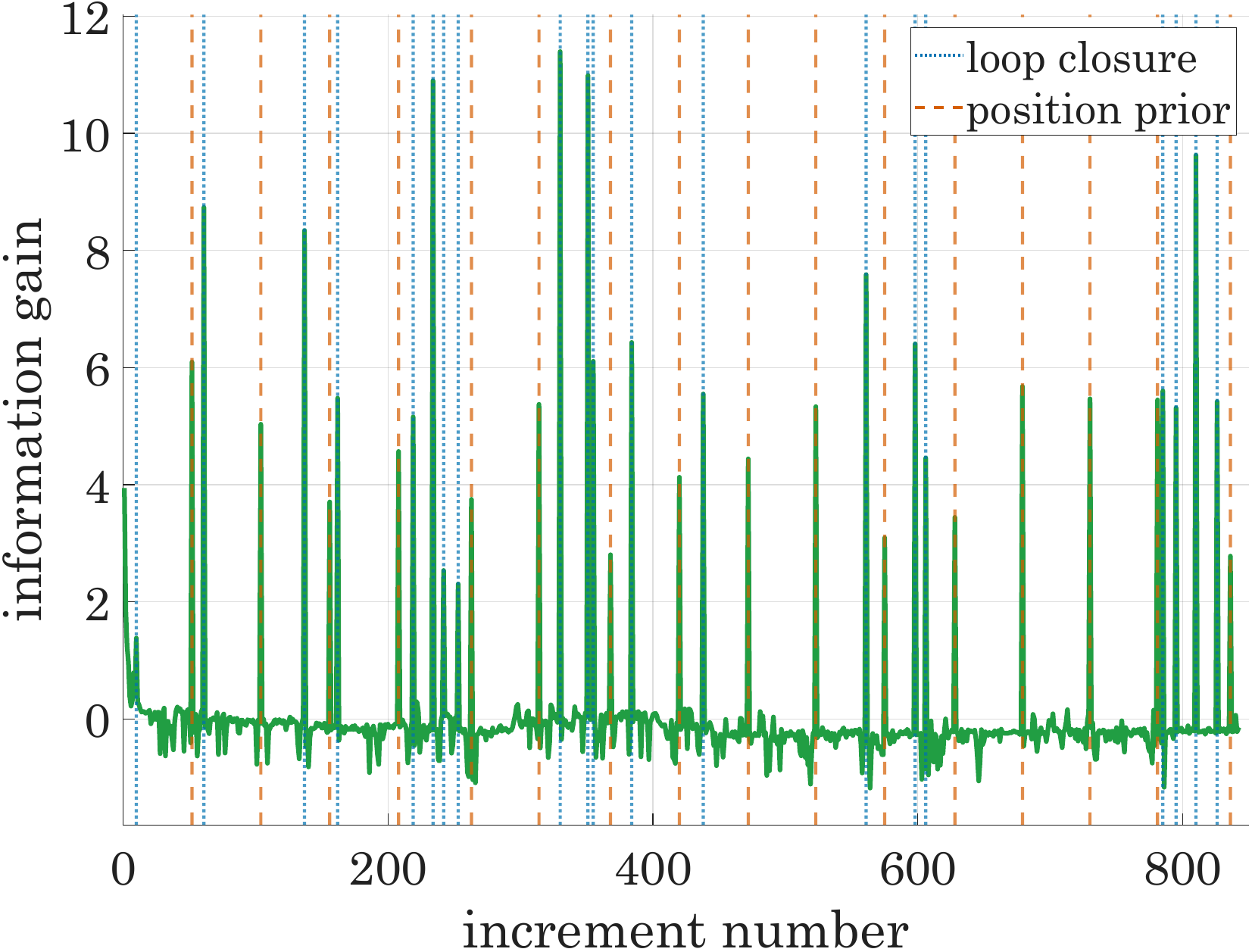}
    \caption{MIT-P}
    \label{fig:info_mitp}
\end{subfigure}
\caption{Information gain over increments for four dataset.}
\label{fig:info}
\end{figure*}

\begin{figure*}
\centering
\begin{subfigure}{0.5\linewidth}
    \includegraphics[scale=.27]{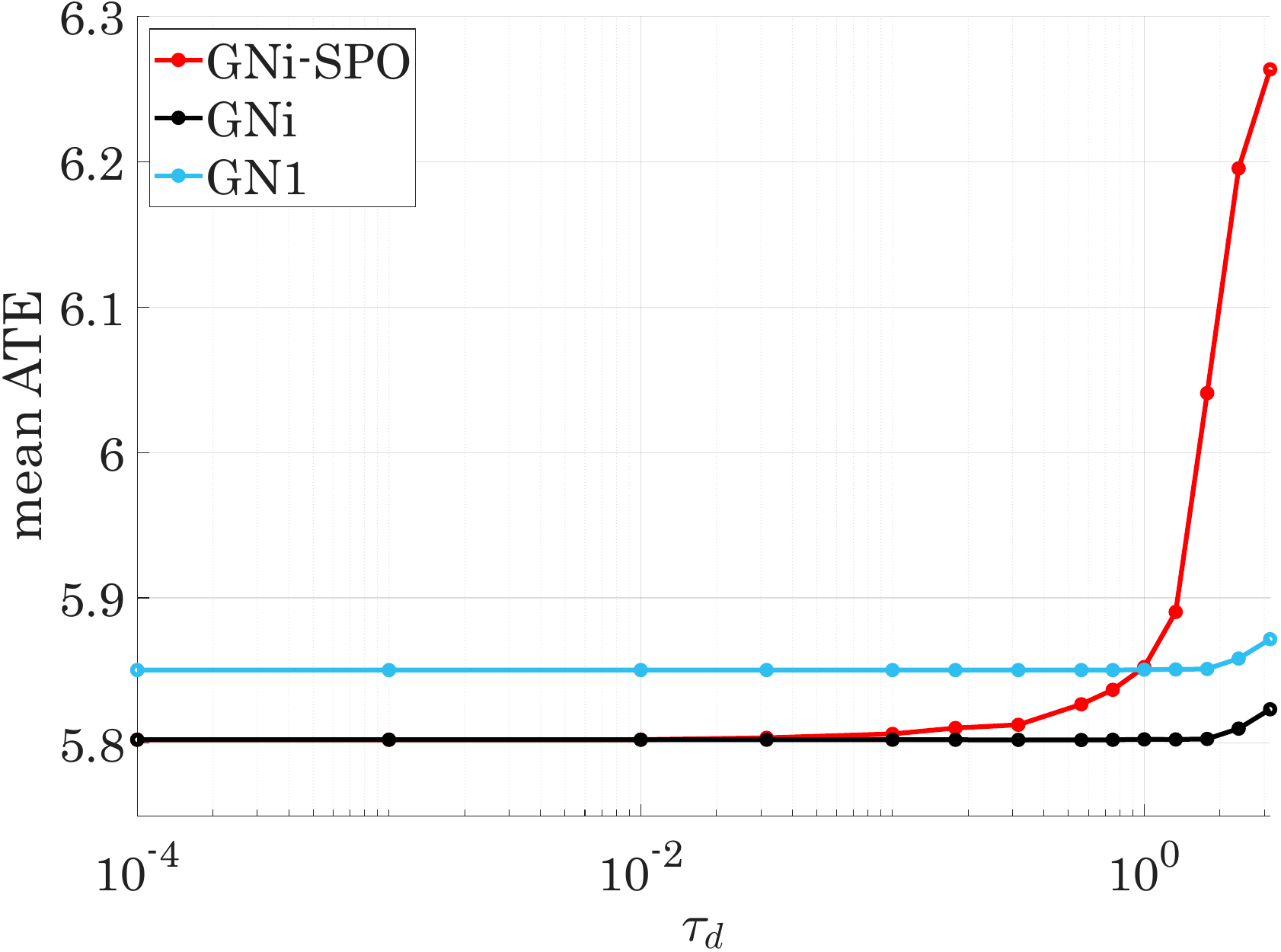}
    \caption{Mean ATE versus $\tau_d$.}
    \label{fig:ATE_d}
\end{subfigure}\hfill
\begin{subfigure}{0.5\linewidth}
    \includegraphics[scale=.27]{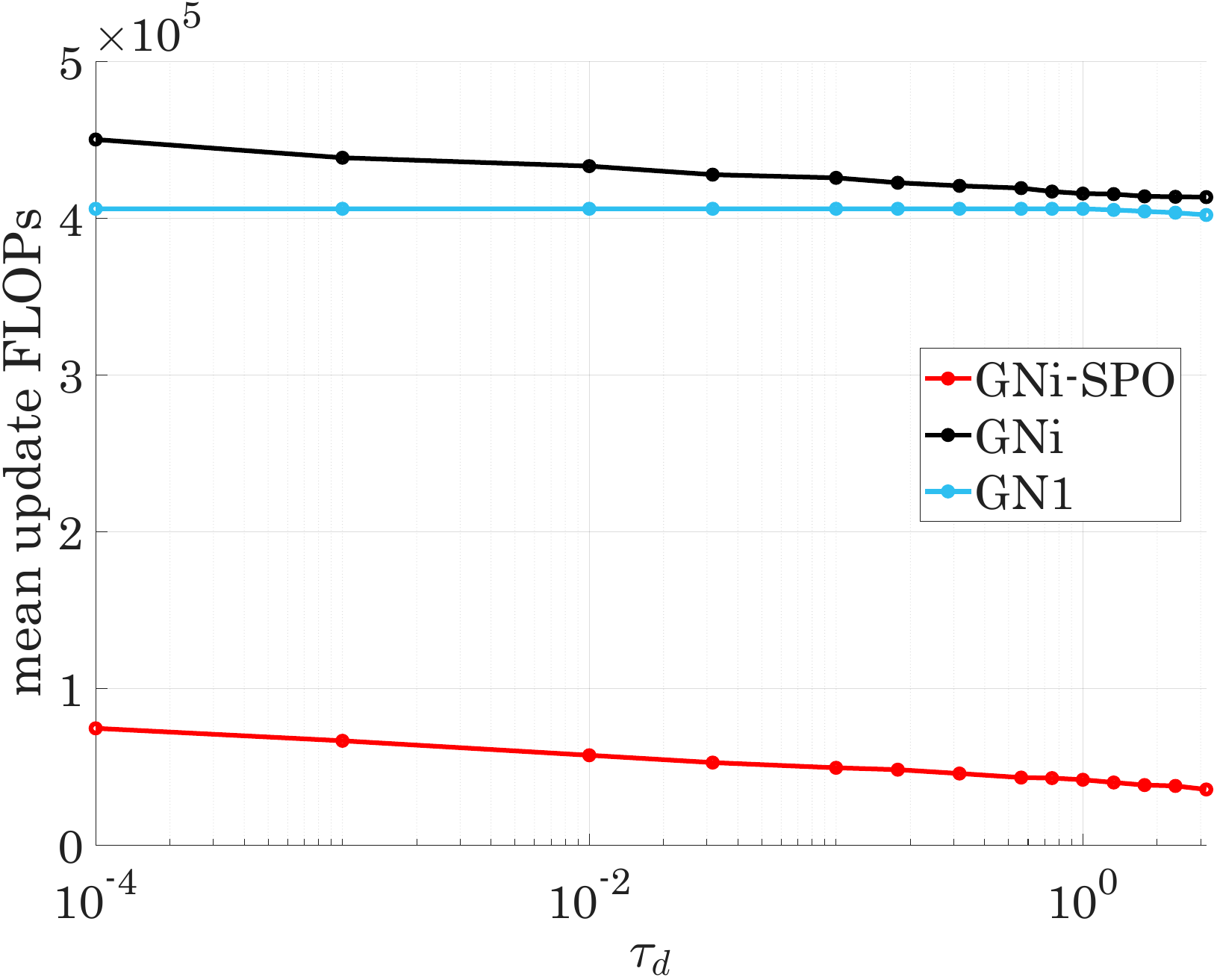}
    \caption{Mean update FLOPs versus $\tau_d$.}
    \label{fig:update_d}
\end{subfigure}
\caption{Performance of the proposed algorithm on the MIT dataset for different values of the update threshold $\tau_{d}$.}
\label{fig:vs_d}
\end{figure*}

\begin{figure*}
\centering
\begin{subfigure}{0.5\linewidth}
    \includegraphics[scale=.27]{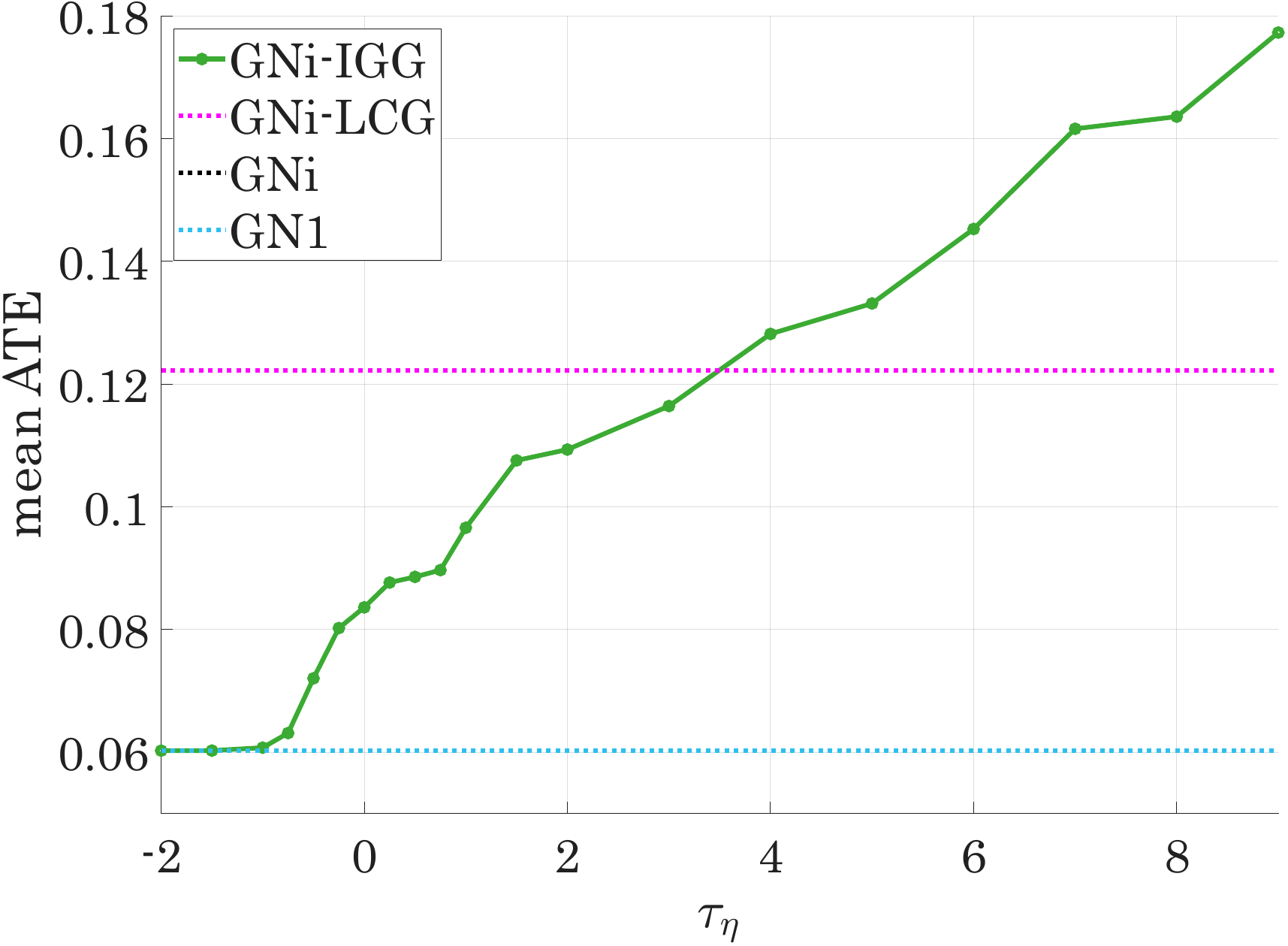}
    \caption{Mean ATE versus $\tau_\eta$.}
    \label{fig:ATE_eta}
\end{subfigure}\hfill
\begin{subfigure}{0.5\linewidth}
    \includegraphics[scale=.27]{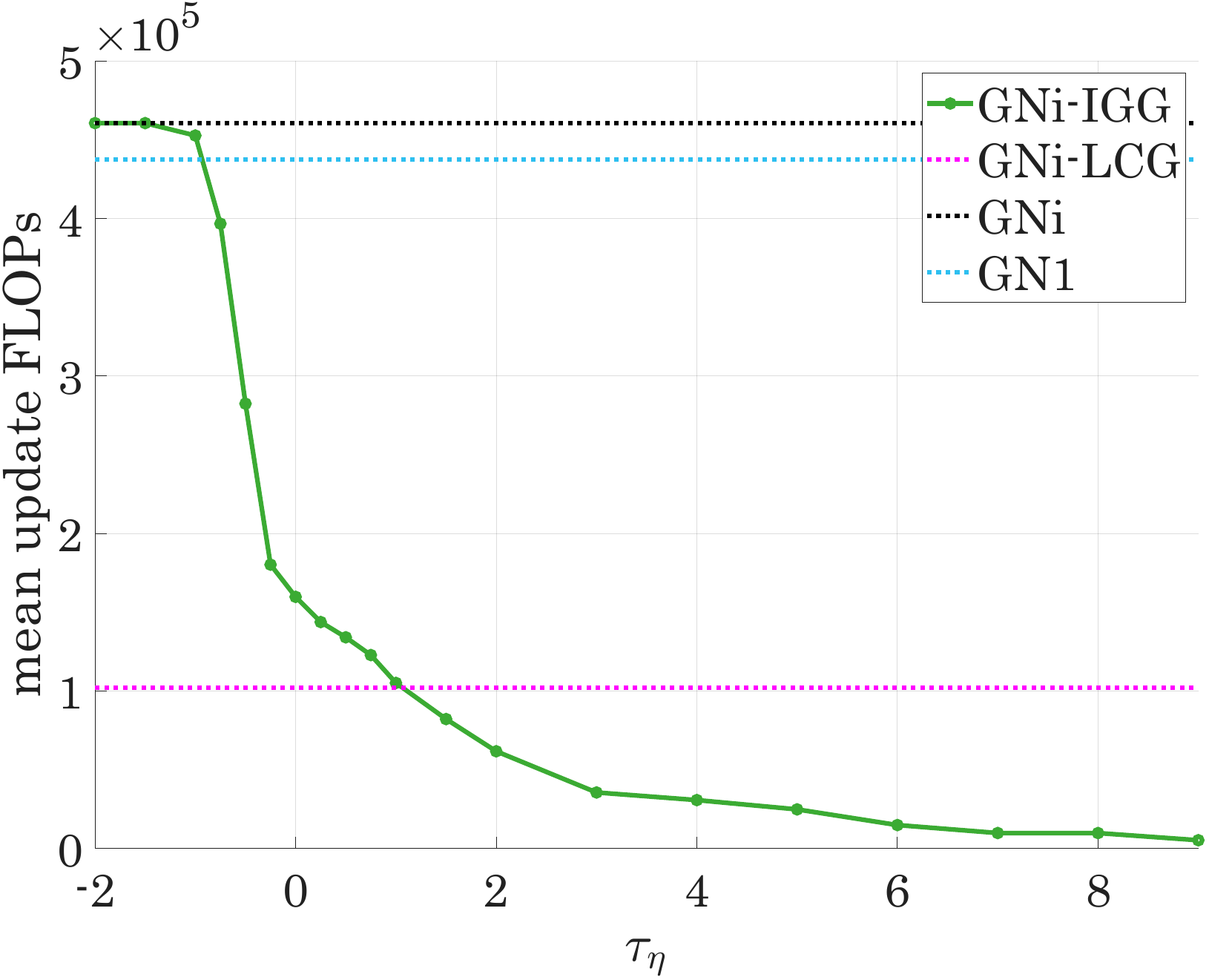}
    \caption{Mean update FLOPs versus $\tau_\eta$.}
    \label{fig:update_eta}
\end{subfigure}
\caption{Performance of the proposed approach on the FR079 dataset for different values of the information-gain threshold $\tau_{\eta}$.}
\label{fig:vs_eta}
\end{figure*}

\subsection{Discussion}

From the results in Table~\ref{tab:results} and Fig.~\ref{fig:ate}, several consistent patterns emerge across datasets of different scale and structure. The first and most important trend is that the proposed GNi-SPO-IGG achieves accuracy essentially indistinguishable from full GNi while requiring far less computation. Both final and mean N$\chi^2$ and ATE values closely track those of the batch-equivalent baseline, demonstrating that selectively solving for and relinearizing only affected variables, when guided by the information-gain criterion, preserves estimation quality. For instance, on the FR079 and FRH datasets, the discrepancies in error between GNi-SPO-IGG and GNi are negligible, yet the proposed approach reduces solve and update FLOPs significantly. This confirms that careful gating and partial optimization can provide substantial efficiency without accuracy degradation.

The efficiency gains are particularly pronounced in larger and denser problems such as CSAIL, Intel, and FRH. As shown in Fig.~\ref{fig:flops}, the computational cost of always performing full updates grows rapidly in these cases, whereas GNi-SPO-IGG effectively contains this growth by adaptively restricting updates to variables most affected by new information. This yields a two- to eight-fold reduction in computation compared to GNi, while both final and mean errors remain essentially unchanged. Even in smaller-scale datasets such as MIT and FR079, where graphs are less dense and updates relatively inexpensive, the proposed method still delivers notable efficiency improvements without sacrificing accuracy. These results suggest that the scalability advantage of the proposed approach will become even more significant as problem size increases.

A comparison between GNi-SPO-IGG and GNi-SPO-LCG further illustrates the value of information-guided gating. In datasets such as MIT, where most of the informative content arises from a small number of loop-closure measurements, both approaches attain similar accuracy and complexity. However, when substantial information is introduced through non-loop-closure measurements, such as position priors in MIT-P, GNi-SPO-LCG fails to exploit this information effectively. By contrast, GNi-SPO-IGG leverages such priors through its information-gain-based gating, achieving markedly better accuracy while maintaining comparable computational cost. This demonstrates the broader advantage of IGG: its ability to incorporate diverse sources of information beyond loop closures, ensuring consistent accuracy across heterogeneous sensing scenarios and enhancing scalability to real-world, long-term SLAM deployments.

Approaches based solely on gating (GNi-LCG and GNi-IGG) exhibit significant instability. While their final error values can sometimes appear acceptable, mean errors often increase by orders of magnitude, particularly on the MIT-P and Intel datasets. GNi-IGG performs only marginally better than GNi-LCG, yet still produces substantially larger errors than when SPO is employed. This highlights the limitations of gating alone, whether based on loop closures or information gain, which cannot adequately respond to incremental odometry or prior information, leaving large portions of the trajectory insufficiently corrected. The outcome is poor trajectory quality for much of the run, even if later corrections at loop closures or high-information events produce seemingly reasonable final results. These findings underscore that loop closures or information gain alone are insufficient for reliable incremental smoothing, and that continuous incremental corrections are essential.

The GN1 baseline, which mirrors iSAM2’s single-GN-iteration strategy, provides another instructive comparison. While GN1 often achieves final accuracy comparable to GNi, it does so at the cost of significantly elevated mean errors. For example, on the Intel dataset, GN1 attains a similar final error to GNi but with mean errors two orders of magnitude higher, while also requiring more computations compared to all considered approaches except GNi. Its poor intermediate accuracy on some datasets can be attributed to its inability to converge within a single GN iteration per increment, particularly when successive high-information measurements arrive in quick succession. In such cases, GN1 cannot keep pace with the influx of information, leaving large errors uncorrected until much later. These observations highlight the importance of adaptivity not only in deciding \emph{when} and \emph{what} to update, but also in determining \emph{how much} to update at each increment.

Trajectory-level analysis, as illustrated in Fig.~\ref{fig:ate}, reinforces these observations. Algorithms incorporating SPO maintain ATE values comparable to GNi across the increments, while gating-only methods (GNi-LCG and GNi-IGG) exhibit spiking or oscillatory error patterns whenever increments contain only low-information odometry measurements. On the synthetic MIT-P dataset, which includes intermittent external priors, these differences are even more pronounced: GNi-LCG and GNi-IGG exhibit very high intermediate errors, with mean values several orders of magnitude larger than those of the other algorithms, whereas the remaining approaches maintain low final errors and reasonable mean error performance. Notably, GNi-SPO-IGG achieves accuracy similar to GNi but at significantly lower computational cost compared to GN1, GNi, and GNi-SPO, demonstrating robustness to external information while avoiding wasted computation. Although GNi-SPO-LCG incurs somewhat lower computational cost than GNi-SPO-IGG, its final and mean errors are substantially higher, owing to its failure to trigger highly beneficial global updates when informative but non-loop-closure position priors arrive. This underscores a key design lesson: information-gain-based gating is essential for effectively leveraging diverse external cues, such as GNSS, UWB, or vision-based priors, ensuring not only accuracy with minimal computation, but also scalability to long-term, multimodal SLAM deployments in real-world environments.

The results presented in Table~\ref{tab:results} and Figs.~\ref{fig:ate} and~\ref{fig:flops} correspond to settings where GNi-SPO-IGG achieves accuracy comparable to the baseline GNi. However, as shown in Figs.~\ref{fig:ATE_d} and~\ref{fig:ATE_eta}, the proposed approach provides a tunable trade-off between accuracy and computational cost. By adjusting the thresholds $\tau_d$ and $\tau_\eta$, one can balance estimation accuracy against computational efficiency according to application requirements.

\section{Concluding Remarks}

We presented a novel incremental SLAM optimizer that achieves a principled balance between accuracy and efficiency by combining an information-theoretic update trigger with selective partial optimization. The key idea is to monitor the change in system entropy (approximated via the log-determinant of the information matrix) induced by new measurements and to use this to decide whether to perform a global update or restrict attention to local variables. When a global update is triggered, GN iterations are executed to convergence, but each iteration is confined to the subset of variables most affected by the recent measurements. This selective update strategy exploits the conditional independence structure of SLAM and yields solutions nearly as accurate as batch optimization while requiring substantially fewer updates and significantly lower computational cost.

Extensive experiments on diverse SLAM benchmark datasets demonstrate that the proposed approach matches the accuracy of batch and full incremental solvers while consistently achieving large computational savings. In particular, its final and mean normalized chi-squared and absolute trajectory errors remain essentially identical to those of GNi, while solve and update FLOPs are reduced by factors of two to eight, especially in large-scale datasets such as CSAIL, Intel, and FRH. The information-gain-based gating strategy generalizes loop-closure detection by capturing all high-impact events, including loop closures, dense reobservations, and external priors such as GNSS updates, while simultaneously filtering out low-information increments that contribute little beyond increased computational burden. The results also highlight the robustness of the proposed approach to external priors: unlike loop-closure-only approaches, our algorithm can seamlessly exploit intermittent position priors without destabilizing the trajectory estimate.

A central feature of our approach is the selective update of only those variables that remain unconverged during GN iterations. Following highly informative measurements, such as major loop closures, typically only a small fraction of the state variables requires further refinement in subsequent iterations. SPO exploits this by restricting computation to the active subset, thereby achieving substantial cost savings without sacrificing accuracy. By contrast, GNi, the full-update algorithm with no selectivity, converges to the same solution but incurs several times the computational cost during large re-optimizations.

In relation to iSAM2, one of the most widely recognized incremental SLAM back-ends, we emphasize an important conceptual and practical distinction. The core of iSAM2's fluid relinearization strategy is a heuristic, meaning it decides to relinearize a variable if the magnitude of its update exceed a threshold. While this can reduce computational cost by avoiding some recomputations, it also carries the risk of accumulating errors if small, but frequent, updates are repeatedly ignored. In contrast, our approach ensures consistent linearizations and prevents drift by always relinearizing any variable that is updated. We apply thresholds to the update step itself, determining if a variable requires further optimization, rather than to the relinearization process as iSAM2 does. Our unshown experiments indicate that iSAM2's thresholds often need to be set very conservatively to prevent the system from diverging, which can make it particularly sensitive to poor initializations. Furthermore, unlike our approach, iSAM2 lacks both gating and partial optimization. It still solves the entire global linear system at every increment, regardless of which variables are actually updated. Our approach, as established in Section~\ref{sec:theory}, is grounded in a more robust framework. It strategically combines information-guided gating with selective partial optimization to achieve both high accuracy and computational efficiency.

Although our current implementation is in MATLAB/Octave and intended primarily for clarity and reproducibility, the algorithmic FLOP counts and structural complexity analysis strongly indicate that a C++ implementation would yield substantial runtime gains. In fact, the proposed framework is fully compatible with established SLAM libraries such as g2o and GTSAM, where fast factor reuse (e.g., via Bayes trees) could be combined with our selective multi-iteration optimization to achieve both scalability and real-time performance. We make our MATLAB/Octave code publicly available\footnote{The code and data used to produce the results in this paper are available at \url{github.com/Reza219/Incremental_SLAM}.} to ensure reproducibility, while noting that a C++ integration is a straightforward engineering extension and an important direction for future work.

The threshold parameters $\tau_\eta$ and $\tau_d$ were set empirically, but future work can explore adaptive schemes based on statistical criteria (e.g., significance tests on normalized chi-squared error growth). Additional possible future directions include incorporating robust cost functions (e.g., dynamic covariance scaling or switchable constraints) to improve resilience to outlier, applying the approach to fixed-lag smoothing where information gain can also inform safe marginalization, and integrating the same framework in active SLAM, where information gain can guide both update scheduling and exploration decisions in resource-constrained scenarios.

In summary, we have developed an efficient and principled incremental SLAM solver that delivers batch-level accuracy at a fraction of the computational cost. By bridging the gap between batch and incremental optimization, the proposed method enables accurate real-time SLAM in computationally constrained environments. It provides a solid foundation for high-precision mapping and localization on embedded platforms, and can be naturally extended to applications in 3D SLAM, multi-robot mapping, and adaptive perception.

\bibliographystyle{IEEEtran}
\bibliography{references}

\end{document}